\documentclass[twoside]{article}

\usepackage[accepted]{aistats2023}

\usepackage[utf8]{inputenc} % allow utf-8 input
\usepackage[T1]{fontenc}    % use 8-bit T1 fonts
\usepackage{hyperref}       % hyperlinks
\usepackage{url}            % simple URL typesetting
\usepackage{booktabs}       % professional-quality tables
\usepackage{amsfonts}       % blackboard math symbols
\usepackage{nicefrac}       % compact symbols for 1/2, etc.
\usepackage{microtype}      % microtypography
\usepackage[dvipsnames]{xcolor}
\usepackage{amsthm} % Theorems Lemmas
\usepackage{amsmath}
\usepackage{bm} % Bold greek letters
\usepackage{mathtools} % Ceils and Floor Macros
\usepackage{amssymb}
\usepackage{graphics}
\usepackage[normalem]{ulem} % for stout
\usepackage{dsfont} % for mathds
\usepackage{authblk}
\usepackage{caption}
\usepackage{algorithm}
\usepackage{algorithmic}
\usepackage{subcaption}
\usepackage{graphicx}

\usepackage[round]{natbib}

\hypersetup{
    colorlinks = true,
    citecolor = blue,
    linkcolor = blue
}

\newcommand{\indep}{\perp \!\!\! \perp}
\newcommand{\0}{\mathbf{0}}

\newcommand{\adj}{*}

\newcommand{\Abf}{\mathbf{A}}
\newcommand{\Bbf}{\mathbf{B}}

\newcommand{\bbar}{b}

\newcommand{\Ccal}{\mathcal{C}}

\newcommand{\Dbf}{\mathbf{D}}
\newcommand{\Dcal}{\mathcal{D}}
\newcommand{\Deltabf}{\mathbf{\Delta}}
\newcommand{\e}{\mathbf{e}}
\newcommand{\E}{\mathbb{E}}
\newcommand{\etab}{\pmb{\eta}}

\newcommand{\Expect}{\operatorname{\mathbb{E}}}

\newcommand{\Fbar}{\overline{\mathcal{F}}}

\newcommand{\Ind}{\mathbb{I}}

\newcommand{\lambdatilde}{\widetilde{\lambda}}
\newcommand{\lambdamin}{\lambda_{\min}}
\newcommand{\lambdamax}{\lambda_{\max}}

\newcommand{\Nold}{n}
\newcommand{\ntrain}{\Nold}

\newcommand{\bigO}{\mathcal{O}}

\newcommand{\Pro}{\mathbb{P}}

\newcommand{\reals}{\mathbb{R}}
\newcommand{\rbar}{\overline{r}}
\newcommand{\rn}{\overline{r}_n}

\newcommand{\Rbar}{\overline{R}}
\newcommand{\Sigmabf}{\mathbf{\Sigma}}
\newcommand{\Sigmahat}{\widehat{\Sigmabf}}
\newcommand{\Sigmabar}{\overline{\Sigmabf}}
\newcommand{\Sigmabarhat}{\widehat{\overline{\Sigmabf}}}
\newcommand{\sigmamin}{\sigma_{\min}}
\newcommand{\sigmamax}{\sigma_{\max}}
\newcommand{\Sb}{\mathbb{S}}

\newcommand{\ubf}{\mathbf{u}}

\newcommand{\bu}{\mathbf{u}}
\newcommand{\bv}{\mathbf{v}}

\newcommand{\vbf}{\mathbf{v}}

\newcommand{\w}{\mathbf{w}}

\newcommand{\wbar}{\mathbf{\overline{w}}}
\newcommand{\what}{\mathbf{\widehat{w}}}
\newcommand{\W}{\mathbf{W}}
\newcommand{\What}{\widehat{\W}}
\newcommand{\whatbar}{\widehat{\overline{\mathbf{w}}}}
\newcommand{\wtilde}{\mathbf{\widetilde{w}}}
\newcommand{\x}{\mathbf{x}}

\newcommand{\X}{\mathbf{X}}
\newcommand{\Xcal}{\mathcal{X}}

\newcommand{\xbar}{\overline{\mathbf{x}}}

\newcommand{\y}{\mathbf{y}}

\newcommand{\A}{\mathbf{A}}

\newtheorem{lemma}{Lemma}
\newtheorem{assumption}{Assumption}
\newtheorem{theorem}{Theorem}
\newtheorem{corollary}{Corollary}
\newtheorem{remark}{Remark}
\newtheorem{proposition}{Proposition}
\newtheorem{definition}{Definition}

\newcommand{\norm}[1]{\left\lVert#1\right\rVert}
\newcommand{\opnorm}[1]{\left\lVert#1\right\rVert_{\rm op}}

\newcommand{\nucnorm}[1]{\left\lVert#1\right\rVert_{\ast}}

\DeclarePairedDelimiter\ceil{\lceil}{\rceil}
\DeclarePairedDelimiter\floor{\lfloor}{\rfloor}

\usepackage[textwidth=2.0cm, textsize=tiny]{todonotes} % for writing

\begin{document}

\twocolumn[
    \aistatstitle{Multi-task Representation Learning with Stochastic Linear Bandits}

\aistatsauthor{
    Leonardo Cella \\ 
    CSML, Italian Institute of Technology
    \And 
    Karim Lounici \\ CMAP, Ecole Polytechnique
    \AND  
    Gr\'egoire Pacreau \\ CMAP, Ecole Polytechnique
    \And 
    Massimiliano Pontil \\
    CSML, Italian Institute of Technology \\ 
    \& Dept of Computer Science, UCL
}

\aistatsaddress{}
]

\begin{abstract}
    \noindent 
    We study the problem of transfer-learning in the setting of stochastic linear contextual bandit tasks. We consider that a low dimensional linear representation is shared across the tasks, and study the benefit of learning the tasks jointly.
    %and study the benefit of learning this representation from in the multi-task learning setting. 
    Following recent results to design Lasso stochastic bandit policies, we propose an efficient greedy policy based on trace norm regularization. It implicitly learns a low dimensional representation by encouraging the matrix formed by the task regression vectors to be of low rank. Unlike previous work in the literature, our policy does not need to know the rank of the underlying matrix, nor {does} it requires the covariance of the arms distribution to be invertible. We derive an upper bound on the multi-task regret of our policy, which is, up to logarithmic factors, of order  $T\sqrt{rN}+\sqrt{rNTd}$, where $T$ is the number of tasks, $r$ the rank, $d$ the number of variables and $N$ the number of rounds per task. We show the benefit of our strategy 
    %in comparison to the 
    over an independent task learning baseline,
    %obtained by solving each task independently, 
    which has a worse regret of order $T\sqrt{dN}$. We also argue that our policy {is minimax optimal} and, when $T\geq d$, has a multi-task regret which is comparable to the regret of an oracle policy which knows the true underlying representation.
    %and learns the task independently.
    %\sout{We also provide a lower bound to the multi-task regret. Finally, we corroborate our theoretical findings with preliminary experiments on synthetic data.}
    \end{abstract}

    \section{INTRODUCTION}
Contextual bandits \citep{abbasi2011improved, li2010contextual, auer2002using}
are a prominent learning framework to study sequential decision problems with partial feedback.
They find applications in numerous fields, ranging from recommender systems \citep{li2010contextual}, to finance \citep{shen2015portfolio} and to adaptive routing \citep{awerbuch2008online}, among others. This methodology was originally motivated by applications in clinical trials \citep{woodroofe1979one}, whereby a doctor has to decide which among available treatments is best suited for a patient, through a sequence of trials. %Similarly, in recommendation systems the objective becomes to understand which product, among a catalog of alternatives, should be suggested to a customer.
A fundamental aspect in bandit problems is to control the trade-off between \textit{exploration} and \textit{exploitation}, namely, the balance between the need of acquiring more information and the temptation to act optimally according to the already available knowledge.

\noindent 
In this paper we study multi-task learning with stochastic linear contextual bandit tasks \citep{lowRank_GLM,li2017provably,filippi2010parametric}. 
Within this setting, 
%a learning agent faces multiple stochastic linear contextual bandit tasks simultaneously. 
each task is associated with a regression vector and proceeds  sequentially. 
At each trial an agent observes a set of different alternatives (%usually referred to as 
\textit{arms}) linked to a feature ({\em context}) vector. 
%For each task, the 
The agent then selects one context vector and subsequently observes a stochastic reward generated by a noisy linear regression associated to the chosen vector. The goal is to design an algorithm (policy) that learns, interacting with the tasks, how to select contexts that are most aligned with the underlying task (regression vector), hence maximizing the cumulative rewards across all the tasks.

\noindent A central idea of multi-task learning is to leverage similarities between tasks in order to facilitate learning. In this paper, we consider 
%the case where 
that the tasks share a 
%there exists a common 
low dimensional representation,
%shared among the tasks 
that is, the task regression vectors span a low dimensional subspace. The benefit of learning such representation has been widely investigated in both the standard supervised learning and in the reinforcement learning settings \citep[see][and references therein]{lounici2011oracle,koltchinskii2011nuclear,negahban2011estimation,calandriello2014sparse}. In the bandit setting, this problem 
%studying this benefit 
presents additional difficulties, since contexts vectors are no longer independently distributed. Indeed, they are 
collected sequentially depending on past observations and on the adopted bandit policy. This entails two main challenges. On the one hand, we would like the collected contexts to span the whole feature space, as this would facilitate the estimation of the unknown regression vectors. On the other hand, collecting contexts which are misaligned with respect to the regression vector results in poor performance for the bandit strategy. 

{\bf Contributions.} 
In this work we present an efficient policy based on the trace norm regularization estimator \citep[see][and references therein]{koltchinskii2011nuclear,negahban2011estimation,pontil2013excess} which leverages the tasks common representation to improve learning.
We provide oracle inequalities for this policy under the restricted strong convexity condition with correct and explicit dependencies (Lemma \ref{Le:OracleInequality} and Proposition \ref{prop:subgaussian2}) that are based based on a novel martingale concentration argument. 
%Section \ref{Sec:LowRankOracleInequality}). 
%Note that this result is . Second, 
Next, 
%we propose a regret-minimization policy that does not require random plays to get i.i.d. data or to know the rank value $r$ (Section \ref{Sec:TransferLearningLinBan}) 
we provide an upper bound for the proposed policy (Theorem \ref{Th:MTLRegretBound}) which is valid without boundedness assumption on the arms and is minimax optimal (up to a logarithmic factor). We noticed that when the number of tasks is larger than the ambient dimension our policy is comparable to the oracle policy which knows the true underlying representation a-priori. A key novelty of the proposed policy over the state-of-the-art \citep{yang2020impact} is that it does not need to know the rank of the representation matrix, it does not require the arms covariance distribution to be invertible and its regret bound is non-trivial already when the time horizon is $O(d)$ in the worse case  and potentially $O(\log d)$ in favorable scenarios.

{\bf Notation.} For a real vector $\x\in\reals^d$, we use $\norm{\x}$ to denote its Euclidean norm. %We use $\Bb^d := \{\x\in\reals^d:\norm{\x}\leq1\}$ and $\Sb^d := \{\x\in\reals^d:\norm{\x}=1\}$ to respectively denote the unit ball and unit sphere in $\reals^d$. 
Given a pair of symmetric matrices $\Abf, \Bbf \in \reals^{d \times d}$, the expression $\Abf \succeq \Bbf$ means that $\Abf-\Bbf$ is positive semi-definite. We respectively use $\lambdamin(\Abf)$ and $\lambdamax(\Abf)$ to refer to the minimum and maximum eigenvalues of a square symmetric matrix $\Abf$. 
Similarly, we respectively use $\sigmamin(\Abf)$ and $\sigmamax(\Abf)$ for the smallest and the largest singular values of a generic matrix $\Abf$.
Given a positive definite matrix $\Abf \succ \0$, we indicate with $\norm{\x}_{\Abf} = \sqrt{\x^\top \Abf \x}$ the corresponding weighted Euclidean norm. For any matrix $\Abf\in\reals^{d\times T}$, we let $\nucnorm{\Abf} = \sum_{i=1}^d\sigma_i(\Abf)$ be the trace norm (sum of its singular values), $\norm{\Abf}_{\rm F}$ the Frobenius norm and $\opnorm{\Abf}=\sigma_{\max}(\Abf)$ the operator norm (largest singular value). We denote with $[\Abf]^i$ its $i$-th row and with $[\Abf]_j$ its $j$-th column. Finally, we use $\text{diag}(\lambda_1,\dots,\lambda_d)$ for the $d \times d$ diagonal matrix with values $\lambda_1,\dots,\lambda_d$ on the diagonal. Given a random event $\Upsilon$, we denote with $\Upsilon^c$ its complement. Finally, given a scalar $\epsilon\in\reals$ we denote with $\mathbf{\epsilon}_d$ the $d$-dimensional vector having value $\epsilon$ in each component. Additional notation is introduced on the way; a table summarizing the notation used throughout the paper is reported in the appendix.

%{\bf Related Works.} 
\section{RELATED WORKS}

In the last two decades many efforts have been devoted to designing contextual bandit policies \citep{abbasi2011improved,abbasi2012online,ariu2020thresholded,auer2002using,bastani2020online,chu2011contextual,kim2019doubly,li2017provably,oh2020sparsity,wang2018minimax,kuzborskij19a,foster2020beyond}. When restricting to the high-dimensional setting an appealing approach is based on sparse linear models, i.e. the number $s$ of non-zero components of the regression vector is assumed to be much smaller than the input dimension $d$. Many different strategies have been investigated \citep{kim2019doubly,ariu2020thresholded,bastani2020online}. Among the proposed approaches, one of the most recent and valuable works is \citep{oh2020sparsity}, where %Authors have 
they design a greedy policy 
%suitable to the generalized linear bandit setting which is 
based on the Lasso estimator. 
%The novelty of 
Interestingly, their approach does not require knowledge of 
%one to know the value 
$s$ and it does not perform random pulls in order to have i.i.d. data. Inspired by this work, we propose a greedy algorithm which does not need to know the rank index $r$ (dimensionality of the underlying representation). A technical challenge that we are facing is how to obtain an accurate estimator relying on non i.i.d. samples, while considering a more complex matrix estimator. We observe that their regret bound argument is not accurate as there might be a hidden dependency in the number of features (the same inaccuracy can be found in~ \citep{cellaSparseMTL, kim2019doubly, calandriello2014sparse}. %\textcolor{red}{\sout{and we provide a matching lower bound under the \textit{restricted eigenvalue conditions}}}. 

%\MP{this para is very long and we consdier splitting into two parts}
Multi-task and meta-learning frameworks have been studied primarily in the supervised-learning setting~\citep{tripuraneni2021provable, denevi2018incremental,denevi2019online,lounici2011oracle,argyriou2008convex,baxter2000model}. Specifically, the impact of representation learning with trace-norm regularization has been widely investigated when considering i.i.d. data \citep{koltchinskii2011nuclear,negahban2011estimation,pontil2013excess}. More recently different authors have investigated the combination of multi-task and meta-learning with interactive learning settings (e.g. bandits and reinforcement learning) \citep{hu2021near,basu2021no,cellaSparseMTL,kveton2021meta,simchowitz2021bayesian,calandriello2014sparse,d2019sharing,cella2020,yang2020impact}. Among the latter category, the most relevant works are \citep{yang2020impact,hu2021near,cellaSparseMTL}. \cite{cellaSparseMTL} considered both the multi-task and the meta-learning frameworks but assuming the task vector parameters to be \textit{jointly sparse}, which is more  restrictive than the low rank assumption considered here. 
Differently, \citep{yang2020impact} considers a similar multi-task problem both in the finite and infinite action settings, whereas \citep{hu2021near} considers only the latter setting. Note that these two references impose stronger conditions and assumptions. 
%is also similar to the one we considered in this work. 
%\sout{Yet}
Indeed, the policies considered in those papers 
%\KL{no longer true}\sout{\karim{are statistically suboptimal},} 
require knowledge of the low-rank parameter $r$ and are not sample efficient. 
%\KL{lack of guarantees for computational efficiency for their policies}. 
In this work, we develop a different and more general analysis for the finite action setting. Notably, our algorithm is computationally efficient and does not requires knowledge of low-rank parameter $r$. 
Our idea is similar to the last approach investigated in \citep{calandriello2014sparse}, but their work considers the simpler case where samples are independently distributed. At last we mention that \citep{kveton2017low-rank} studied low-rank stochastic bandits. However they consider a completely different protocol, regret and arms definitions. Thus it is not possible to draw a quantitative comparison to this work.
%Besides the hidden dependency imprecision  mentioned before, we also observe that their proof is invalid since it relies on a erroneous concentration argument. }

\section{PROBLEM SETTING}
In this section we recall the stochastic linear contextual bandit problem and then introduce the multi-task learning framework.
%\footnote{The related meta-learning setting will be introduced and addressed later in the paper.}.

\subsection{Stochastic Linear Contextual Bandits}\label{SubSec:LinearBandits}
The contextual linear bandit problem consists of a sequence of $N$ interactions between a learning policy $\pi$ and the environment. At each round $n\in[N]$, the agent is given a set $\Dcal_n\subseteq \reals^d$ formed of $K$ arms (context vectors) from which it has to select one arm $\mathbf{x}_{n}\in\Dcal_n$ among the available ones. Then, the agent observes the corresponding reward $Y_n$ via  a noisy linear regression,
%which is assumed to be a random variable. When considering linear contextual bandits, the observed reward satisfies
\begin{equation}\label{Eq:LinearReward}
    Y_{n} = \x_n^\top \w + \eta_n,
\end{equation}
where $\eta_n$ is a random variable %taking values in $[-M,M]$ \massi{obsolete, say the noise is aussian]}
whose statistical distribution will be specified later. By knowing the true regression vector $\w$, at each round $n\in[N]$ the optimal policy $\pi^*$ selects $\mathbf{x}^*_n = \arg\max_{\mathbf{x}\in \Dcal_n} \mathbf{x}^\top \w$, maximizing the instantaneous expected reward. The learning objective is then to maximize the cumulative reward, or equivalently, to minimize the \textit{pseudo-regret}
    \begin{equation}
    	R(N,\w) = \sum_{n=1}^N r_n = \sum_{n=1}^{N}(\mathbf{x}^*_n - \mathbf{x}_{n})^\top \w. \nonumber
    \end{equation}
\noindent We make the following additional assumptions on the arms, the regression vector and the noise variables. %\massi{The following assumption speaks only about the arms, no regression vectors!}
\begin{assumption}[Contexts Distribution %and Regression Vector
]\label{Ass:BoundedNorms} %\textcolor{red}{\sout{All arms have bounded norm, that is, for every $\x \in \cup_{n\in[N]}\Dcal_n$ it holds that $\norm{\x}_{\infty} \leq \bbar$.}} 
At each round $n\in[N]$, the decision set $\Dcal_{n}\subset \reals^d$ consists of $K$ $d$-dimensional vectors $\x_k$ admitting representation $\x_k = \Sigma_k^{1/2}\mathbf{z}_k$ where $\Sigma_k$ is the covariance operator of $\x_k$ and $\mathbf{z}_k$ is a $d$-dimensional vector with independent 
%\MP{I'd always use sub-Gaussian throughout} 
sub-Gaussian entries with zero mean and variance $1$. We assume the tuples $\Dcal_{1},\dots,\Dcal_{N}$ to be drawn i.i.d. from a fixed unknown zero mean sub-Gaussian joint distribution $p$ on $\mathbb{R}^{Kd}$. 
%For any integer $k\in [K]$, we denote by $\Sigma_k$ the $d\times d$ covariance matrix of arm $k$. 
% \textcolor{red}
%Finally, we assume that $\norm{\w}_2{\leq}L$, for some constant $L>0$.
\end{assumption}
Let $C_{\mathbf{z}}$ be a positive constant satisfying  
$$
\max_{1\leq k \leq K}
\|\mathbf{z}_k\|_{\psi_2}\,{<}\,C_{\mathbf{z}}.
$$
We denote by $C_{\x}$ the sub-Gaussian Orlicz norm of the $K$ marginal distributions of $p$ corresponding to the $K$ arms, that is 
$
C_{\x} := \max_{1\leq k \leq K}\|\x_k\|_{\psi_2}.
%\leq  C_{\mathbf{z}} \max_{1\leq k \leq K}\left\lbrace \opnorm{\Sigma_k}^{1/2} \right\rbrace <\infty.
$
An obvious computation gives 
\[
C_{\x} \leq  C_{\mathbf{z}} \max_{1\leq k \leq K}\left\lbrace \opnorm{\Sigma_k}^{1/2} \right\rbrace <\infty.
\]

{The above assumption is quite standard when considering high-dimensional linear bandits \citep{bastani2020online,SparseBanditsLowerBound,kim2019doubly,tsitsiklis}. Note that vectors associated to different arms are allowed to be correlated between each other. Moreover, Assumption \ref{Ass:BoundedNorms} implies that each arm $k$ admits zero mean, square-integrable marginal distribution with $d\times d$ covariance $\Sigma_k$.}

Finally, similarly to \citep{bastani2021mostly,oh2020sparsity}, we introduce the following assumption on the arms distribution. This mild condition is necessary for our analysis in order to control the fulfillment of a specific regularity property (the RSC condition, see Definition \ref{Def:RSC} below) by the empirical covariance matrix. %\MP{this last sentence is a bit quick/unclear} 
Specifically, it will allow the designed policy to avoid interleaving its arm selection strategy with random plays.

\begin{assumption}[Arms distribution]\label{Ass:ArmsDistribution} 
There exists a constant $\nu\,{<}\,\infty$ such that $p(-\xbar)/p(\xbar)\,{\leq}\,\nu\;\forall \xbar \in \reals^{Kd}$. 
{Moreover, there exists a constant $\omega_\Xcal<\infty$ such that, for any permutation $(a_1,\dots,a_K)$ of $[K]$, any integer $i\in\{2,\dots,K{-}1\}$ and any fixed vector $\w\in\reals^d$,} it holds that
%Additionally, let us consider a permutation $(a_1,\dots,a_K)$ of $(1,\dots,K)$. For any integer $i\in\{2,\dots,K-1\}$ and any fixed vector $\w\in\reals^d$, there exists a constant $\omega_\Xcal<\infty$ such that
\begin{align*}
    \E&\left[\x_{a_i} \x_{a_i}^\top \Ind\left\{\x_{a_1}^\top \w<{\cdots}<\x_{a_K}^\top\w\right\}\right] \preceq \\
    &\preceq \omega_{\Xcal} \E\left[ \left( \x_{a_1} \x_{a_1}^\top {+} \x_{a_K}\x_{a_K}^\top \right) \Ind\left\{\x_{a_1}^\top \w <\cdots<\x_{a_K}^\top\w\right\}\right].
\end{align*}
\end{assumption}
%\MP{I think the displayed equation requires some more comments/intuition - we can discuss this later and add a sentence}
\noindent Parameter $\nu$ characterizes the skewness of the arms distribution $p$; for symmetrical distributions $\nu = 1$. Notice that Assumption \ref{Ass:ArmsDistribution} is satisfied for a large class of distributions both discrete and continuous (e.g. Gaussian, multi-dimensional uniform and Rademacher). 
The value $\omega_\Xcal$ captures dependencies between arms: the more positively correlated they are, the smaller $\omega_\Xcal$ will be. The extreme scenario is given by perfectly correlated arms, in which case we have $\omega_\Xcal$ independent of any problem parameters. Finally, when arms are generated i.i.d. from a multivariate Gaussian or a multivariate uniform distribution over the sphere we have $\omega_{\Xcal}=O(1)$.

\subsection{Low-Rank Linear Contextual Bandits}\label{SubSec:MTLLinearBandits}

In this paper, we address the problem of simultaneously solving $T$ linear contextual bandit tasks. Each task $t\in[T]$ lasts for $N$ rounds and is associated with a regression vector $\w_t \in \reals^d$. We denote with 
\[
\W = [\w_1,\dots,\w_T]
\]
the $d\times T$ matrix, whose columns are formed by the $T$ regression vectors, which are unknown to the learner. The bandit tasks 
are explored in parallel. At each round $n\in[N]$ 
%\sout{face each task sequentially from the first up to the last}. 
each task simultaneously receives a decision set $\Dcal_{t,n}$ from the environment, upon which the learner (policy $\pi$) picks an arm $\x_{t,n}$, and the associated feedback $y_{t,n}$ is observed via the linear regression  
\begin{equation}
\label{eq:lin-ttt}
        y_{t,n} = \x_{t,n}^\top \w_t + \eta_{t,n}
\end{equation}
where $\eta_{t,n}$ is a noise random variable which we specify below. 
%\karim{\sout{Only when the reward is observed the protocol moves to the next task $t+1$.}}
%As done in \citep{argyriou2008convex}, In this paper we consider the following sparsity assumption.

We make the following assumption on the task regression vectors. 
\begin{assumption}[Low-Rank Assumption]\label{Ass:Sparsity} The matrix $\W\in\reals^{d\times T}$ has rank $\rho(\W)=r$, with $r\ll \min(d,T)$. 
\end{assumption}
\noindent {The above assumption implies that there exists a low dimensional representation $\Bbf\in\reals^{d\times {r}}$ with orthonormal columns and a matrix $\mathbf{C}\in \reals^{r\times T}$ such that $\W = \Bbf \mathbf{C}$.} 
%\sout{an orthogonal matrix $\mathbf{Q}\in\mathbf{O}^d$ that induces a
%n alternative, {low-dimensional}, representation $\Bbf\in\reals^{d\times \karim{r}}$ such that $\W = {\bf Q} \Bbf$.} 
{This is in-line with the standard high-dimensional setting where many features are redundant.}

\noindent Our principal objective is to minimize the multi-task pseudo-regret,
\begin{eqnarray} \nonumber
   \Rbar({T},{N}) &=& \sum_{t=1}^T R({N},{\w_t}) \\ &=& \sum_{t=1}^T \sum_{n=1}^N (\x_{t,n}^* {-} \x_{t,n})^\top {\w_t},
   \label{Eq:MTLRegret}
\end{eqnarray}
where $\x_{t,n}^*=\arg\max_{\x\in\Dcal_{t,n}}\x^\top {\w_t}$.

\section{LOW-RANK MATRIX ESTIMATION}
\label{Sec:LowRankOracleInequality} 
In the standard supervised learning setting a natural estimator suited to Assumption \ref{Ass:Sparsity} is given by the trace (nuclear) norm regularized estimator \citep[see][and references therein]{argyriou2008convex,buhlmann2011statistics,lounici2011oracle,negahban2011estimation}. In particular a large body of works have shown that bounds on the estimation error is controlled by the rank of the underlying regression matrix.
Here we adapt this methodology to the bandit setting. At each round $n\in[N]$, we estimate matrix $\W$ via $\What_{n+1}$ as the solution of the following trace norm regularization problem
\begin{equation}\label{Eq:TraceNormReg}
   %\What_{n+1}\in 
   \arg\min_{\A\in\reals^{d\times T}} \frac{1}{n} \sum_{t=1}^T \norm{\y_{t,n} - \X_{t,n} [\A]_t }_2^2 + \lambda_n \nucnorm{\A}
\end{equation}
where the design matrix  $\X_{t,n} \in \mathbb{R}^{n\times d}$  
%= [\x^\top_{t,1},...,\x^\top_{t,n}]$ 
contains the context vectors $\x_{t,i} \in \mathbb{R}^d$, $i \in [n]$ as its rows, the vector $\y_{t,n} \in \mathbb{R}^n$ is formed by the rewards for task $t$ after $n$ interactions, sampled from \eqref{eq:lin-ttt}, and $\nucnorm{\A}$ is the trace norm of the matrix $\A$, that is the sum of its singular values. 
If compared to the Lasso estimator, the objective function \eqref{Eq:TraceNormReg} encourages low-rank matrices instead of sparse vectors. Before presenting the technical results we need to introduce the following additional notation.

{\bf Covariance matrices.} We indicate the theoretical {averaged} $d \times d$ covariance matrix as 
\begin{equation}
\label{eq:S-111}
    \Sigmabf = \frac{1}{K}\E\left[\sum_{k=1}^K \x_k \x_k^\top \right] = \frac{1}{K} \sum_{k=1}^K \Sigma_k,
\end{equation} where the expectation is over the decision set sampling distribution $p_\Xcal$ which is assumed to be shared between the tasks. 

For every $t \in [T]$, we denote the empirical covariance matrix for task $t$ as
\begin{equation}
\label{eq:S-222}
\Sigmahat_{t,n}=\frac{1}{n}\sum_{i=1}^n\x_{t,i} \x_{t,i}^\top.
\end{equation}

Moreover, we use the notation $\Sigmabar$ and $\Sigmabarhat_n \in\reals^{dT\times dT}$ for the theoretical and the empirical multi-task matrices respectively. They are both block diagonal and composed by the corresponding $T$ single task $d\times d$ matrices on the diagonal, that is
$$
\Sigmabar = {\rm diag}(\Sigmabf,\dots,\Sigmabf)
$$
and
$$
\Sigmabarhat_n= {\rm diag}(\Sigmahat_{1,n},\dots,\Sigmahat_{T,n}).
$$
\noindent 
{We introduce now the restricted strong convexity (RSC) condition on covariance $\Sigmabar$. To this end, we denote by $\rm{Vec}(\Deltabf)$, the vector in $\reals^{dT}$ obtained by stacking together the columns of $\Deltabf
\in \reals^{d\times T}$. Let 
$$
\W={\bf U}{\bf D}{\bf V}^\top
$$ be the singular value decomposition of matrix $\W$ of rank $r$ where ${\bf U}\in \reals^{d\times r}$, ${\bf V}\in\reals^{T\times r}$ are the matrices formed by the left and right singular vectors, respectively and ${\bf D}$ is the $r\times r$ diagonal matrix of singular values.}

\begin{definition}[RSC Condition]\label{Def:RSC}
We say that the restricted strong convexity (RSC) condition is met for the theoretical multi-task matrix $\Sigmabar\in\reals^{dT\times dT}$, with positive constant $\kappa(\Sigmabar)$ if
\begin{align}
\label{eq:coneCr}
    \min\left\{
    \frac{\norm{\rm{Vec}(\Deltabf)}_{\Sigmabar}^2}{2 \norm{\rm{Vec}(\Deltabf)}_2^2}: \Deltabf \in \Ccal(r)
\right\} \geq \kappa(\Sigmabar),
\end{align}
%\MP{I'd move the comment on $\rm{Vec}(\Deltabf)$ before the definition. I'd also repeat it in the notation section and/or notation table}
where 
\begin{equation}\label{Eq:Cset}
  \Ccal(r){=}  \Big\{ \Deltabf \in \reals^{d\times T}: \nucnorm{\Pi(\Deltabf)} \leq {3} \nucnorm{\Deltabf {-} \Pi(\Deltabf)}\Big\}
\end{equation}
and $\Pi(\Deltabf)$ is the projection onto set 
$$
%\Bcal:=
\Big\{\Deltabf\in\reals^{d\times T}: {\rm{Col}}(\Deltabf)\perp {\bf U}, {\rm{Row}}(\Deltabf)\perp {\bf V}\Big\}.
$$
\end{definition}

{The RSC condition allows us to control the error $\Deltabf_n = \What_n - \W$ as it guarantees that the considered empirical loss is strictly convex on a restricted subset of approximately low rank matrices defined by the cone $\Ccal(r)$.}

\noindent 
\begin{remark}[Value of $\kappa(\Sigma)$] As discussed in \citep{calandriello2014sparse} Definition \ref{Def:RSC} can be compared with the restricted eigenvalue and the compatibility conditions that have been investigated for the group Lasso \citep{buhlmann2011statistics,lounici2011oracle} and Lasso \citep{oh2020sparsity,kim2019doubly,calandriello2014sparse} estimators. {It is standard in high-dimensional statistics or in compressed sensing to assume the existence of an absolute constant $\kappa>0$ such that $\kappa(\Sigmabar)>\kappa>0$. Such conditions are satisfied w.h.p. for instance by context vectors with i.i.d. zero mean, variance 1, sub-Gaussian entries.}
\end{remark}

\subsection{Oracle Inequality with non i.i.d. Data}
We can now state our first result which controls the Forbeniou-norm estimation error 
\[
\Big\|\What_{n+1} - \W\Big\|_{\rm F}  
\]
assuming the RSC condition to hold for the empirical multi-task matrix $\Sigmabarhat_n$. In Proposition \ref{Le:FirstN0} we will show that such condition is satisfied with high probability under our Assumptions \ref{Ass:BoundedNorms}, \ref{Ass:ArmsDistribution} and the RSC condition on $\Sigmabar$.

{{We denote by $\{\e_1,\ldots,\e_T\}$ the standard canonical basis of $\reals^T$.}
We have then to upper bound the operator norm of the following matrix 
\[
\Dbf_n = \sum_{t=1}^T \sum_{i=1}^n \eta_{t,i}\, \x_{t,i}  \e_{t}^\top
\]
in order to set the regularization parameter at round $n$ so that the estimation bound in the following lemma holds with high probability.}

\begin{lemma}\label{Le:OracleInequality} Let $\{\x_{t,i}: t\in[T], i\in[N]\}$ be the sequence generated by Algorithm \eqref{Alg:ASNucRegBan} 
%adapted sequence such that $\x_{t,s}$ may depend on $\{\x_{\tau,i}: \tau \leq t , i \leq s\}$. 
Suppose the RSC condition holds for the empirical multi-task matrix $\Sigmabarhat_{n}$ with constant $\kappa\big(\Sigmabarhat_{n}\big)$. For $\delta\in(0,1)$, define the regularization parameter $\lambda_n=\lambda_n(\delta)$ such that with probability at least $1-\delta$
\begin{equation}
\frac{1}{n} \opnorm{\Dbf_n }\leq  \lambda_n.
 %\sum_{t=1}^T\sum_{i=1}^n \x_{t,i}  \e_{t}^\top \eta_{t,i}
 \label{eq:Regn}
\end{equation}
%\MP{write the rhs in (9) as for ${\bf D}_n$ below}
Then with probability at least $1-\delta$ the trace-norm regularized estimate $\What_n$ defined in (\ref{Eq:TraceNormReg}) satisfies
\begin{equation}\label{Eq:Lemma1Ineq}
    \norm{\What_{n+1} - \W}_{\rm F} \leq \frac{32\lambda_n\sqrt{r}}{\kappa\big(\Sigmabarhat_{n}\big)}.
\end{equation}
\end{lemma}
%\begin{proof}
The proof, which is presented in the appendix, follows along the lines of~\citep[Theorem 1]{negahban2011estimation}.

\begin{remark}
{We note that a similar result can be also be found in \citet{koltchinskii2011nuclear}. In addition, as pointed out in these references, it is possible to extend this result to the case of approximately low rank matrices $\W$ by introducing an additional mispecification error in \eqref{Eq:Lemma1Ineq}.}
\end{remark}
%\end{proof}

\subsection{Controlling the Noise Term}
The result in Lemma \ref{Le:OracleInequality} requires us to set the regularization parameter at round $n$ as in \eqref{eq:Regn}. To this end we exploit deviation bounds for martingales. 
Let us define the filtration $(\Fbar_{n})_{n\geq 0}$ on a probability space $(\Omega,\mathcal{A},\mathbb{P})$ as follows: $\Fbar_{0}$ is the trivial  $\sigma$-field $\{\emptyset,\Omega\}$, and for any $n\geq 1$,
\[
\Fbar_{n}=\sigma\big(\X_{1,i},\eta_{1,i},\dots,\X_{T,i},\eta_{T,i},\, i\in [n]\big).
\]
We consider the standard noise assumption \citep{abbasi2011improved,bastani2020online,cella2020}.
%\massi{[The notation for $\Fbar_{n}$ could be shorter if we use matrix and vector notations: $\Fbar_{n}=\sigma\left(\X_{1,n},\etab_{1,n},\dots,\X_{T,n},\etab_{T,n}\right)$.]}
\begin{assumption}[subGaussian noise]\label{Ass:subGaussnoise} 
%All arms have bounded norm, that is, for every $\x \in \cup_{n\in[N]}\Dcal_n$ it holds that $\norm{\x}\leq b$ \leo{$\norm{\x}_{\infty} \leq 1$}. 
The noise variables $(\eta_{t,n})_{t,n}$ are a sequence of sub-Gaussian random variables adapted to the filtration $\{\Fbar_{n}\}_{n\geq 0}$ and such that for any $1\leq t\leq T$ and  $n\geq 1$,
$$
\E[\eta_{t,n}|\Fbar_{n-1}] = 0,\quad\text{and}\quad \E[\eta_{t,n}^2|\Fbar_{n-1}] \leq \sigma^2,
$$
and $\eta_{1,n},\ldots,\eta_{T,n}$ are mutually independent conditionally on $\Fbar_{n-1}$. We denote by $c_{\eta}$ the sub-Gaussian norm of the $\eta_{t,n}$'s, that is, $\max_{t,n}\{\|\eta_{t,n}\|_{\psi_2}\} \leq c_\eta$.
\end{assumption}

\begin{proposition}[SubGaussian noise]
\label{prop:subgaussian2} Let Assumptions \ref{Ass:BoundedNorms} and \ref{Ass:subGaussnoise} be satisfied. Then, with probability at least $1-\delta$
\[
    \frac{1}{n} \opnorm{\Dbf_n} \leq \lambda_n
\]
where
\begin{align}\label{Eq:RegParameter}
&\lambda_n =  \max_{k\in [K]}\opnorm{\Sigmabf_k}^{1/2} 
\left[\sigma \sqrt{ \frac{(d+T) \, \log \frac{2 N(d+T)}{\delta}}{\ntrain}
}  \right.\notag \\
&\hspace{0.1cm} \bigvee \left. \frac{ C\sqrt{\left(T  + d+ \log \frac{4N}{\delta} \right)
\log^3 \frac{8N(d+T)}{\delta}}}{\ntrain} \,\right]
\end{align}
where $C>0$ can depend only on $c_{\eta}, C_{\mathbf{z}}$.

\end{proposition}
\begin{proof}[Proof Sketch]
The complete proof is given in Appendix \ref{AppSec:ControlDn}. %\ref{AppSec:ControlDn}.
%Here is a short sketch. 
The proof relies on Freedman's inequality for martingale \citep[][Corollary 1.3]{tropp2011freedman} combined with a stopping time argument to handle unbounded arms.

\end{proof}

\section{MULTI-TASK BANDITS VIA TRACE NORM REGULARIZATION}\label{Sec:TransferLearningLinBan}
\begin{algorithm}[t!]\caption{Trace-Norm Bandit}\label{Alg:ASNucRegBan}
\begin{algorithmic}[1]
\REQUIRE Confidence parameter $\delta$, noise variance $\sigma^2$
%\STATE $\What_{1} = {\bf 0}_{d\times T}$, $\lambda_1 = $
\STATE At round $n=1$ arms are picked randomly
\STATE Observe $Y_{1,1},\dots,Y_{T,1}$
\FOR{$n \in 2,\dots,N$} % Rounds
    \STATE update $\What_n$ and $\lambda_n$ according to (\ref{Eq:TraceNormReg}) and (\ref{Eq:RegParameter}) 
    \FOR{$t \in 1,\dots,T$} % Tasks
        \STATE observe $\Dcal_{t,n}$
        \STATE pick $\x_{t,n} \in \arg\max_{\x\in\Dcal_{t,n}} \x^\top [\What_n]_t$
        \STATE observe reward $Y_{t,n}$
    \ENDFOR
\ENDFOR
\end{algorithmic}
\end{algorithm}
\noindent In this section we present our proposed algorithm for multitask learning with linear stochastic bandits. 
%We also comment on its use for meta-learning. 
\noindent The algorithm  relies on the trace-norm regularized estimator in~\eqref{Eq:TraceNormReg} and is displayed in Algorithm \ref{Alg:ASNucRegBan}. It does not require to know any input parameter besides an upper bound on the noise variance $\sigma^2$ and the confidence value $\delta$, which are necessary to specify the regularization parameter\footnote{Using a regularization is quite common in the bandit literature (e.g. \cite{abbasi2011improved,bastani2020online}. 
} $\lambda_n$. Indeed this is the only parameter required by our policy and it should be chosen large enough ({in theory}) as prescribed by Proposition \ref{prop:subgaussian2}; see also equation (11). 
%At each round, Algorithm \ref{} updates the task parameters and the regularization parameter according to 

%Our policy does not need any forced-sampling strategy similarly \citep{oh2020sparsity}. \MP{Gergon "forced-sampling strategy" is a bit unclear} 

Algorithm \ref{Alg:ASNucRegBan} updates the task parameters via a trace-norm regularized estimation using all the already observed data. Similarly to \citep{oh2020sparsity}, our policy is completely greedy and does not need any forced-sampling strategy. 
\noindent
 %Other strategies perform the exploration by playing randomly for a certain number of rounds (see \cite{bastani2020online,kim2019doubly} but they also require the i.i.d. assumption.
{This phenomenon was recently highlighted by \citep{bastani2021mostly} which proved that greedy policies do not need an explicit exploration and can even be rate-optimal if there is sufficient randomness in the observed contexts (relying on i.i.d. arms assumption). 
%\sout{This is unlike the class of OFUL-type policies (see \cite{abbasi2011improved} and references therein)}. 
This is unlike the UCB-like policies (like OFUL) or forced-exploration policies \cite{bastani2020online}. 
%All other Lasso policy perform initial random exploration play 
Our trace-norm bandit policy (Alg. \ref{Alg:ASNucRegBan}) belongs to the former class of “exploration-free” greedy policies. 
%This allowed us to avoid both playing i.i.d. for some rounds or using an UCB like exploration-strategy. 
More precisely, our policy performs a natural initial exploration thanks to the arms randomness condition (Assumptions \ref{Ass:BoundedNorms} and \ref{Ass:ArmsDistribution}) until the empirical covariance matrix satisfies the (RSC) condition. In the regret analysis below this can be seen in the definition of the first $N_0$ rounds in \eqref{eq:N0_simplified}.}

\subsection{Multi-Task Learning}
{In this subsection we present our main result which is a high-probability upper bound on the multi-task regret incurred by the trace-norm bandit policy.}

%\karim{
%We define the effective rank of a covariance matrix $\A$ as $\mathbf{r}(\A) = \mathrm{trace}(\A)/\opnorm{\A}$. Moreover, 
For any $\delta \in (0,1)$, let $N_0(\delta)$ be the smallest integer such that
\begin{align} 
\label{eq:N0_simplified} 
N&\geq  C \max_{k\in [K]} \opnorm{\Sigma_k}^2 \left(   \big(r \log d\big) \log \frac{4TN}{\delta} \right)^2,
\end{align}
where $C>0$ is some large enough numerical constant that may depend only on the distribution of the arms, in particular $\nu ,\omega_{\Xcal}$ and $C_{\mathbf{z}}$. 

Note that \eqref{eq:N0_simplified} is a sufficient condition on the minimum number of rounds for our regret bound to be valid. Interestingly it depends on the ambient dimension $d$ only logarithmically. This means that for our greedy policy, we can guarantee that the duration of the implicit exploration phase is at most $r^2$ (up to logarithmic factors).  In situation where the $\W^*$ matrix is low-rank with $r\ll \sqrt{d}$, this represents another benefit of our trace norm policy procedure over 
\citep{yang2020impact} which requires a much longer exploration phase of at least $d^2$ rounds for their regret bound to be valid.

We now state the main result.
\begin{theorem}\label{Th:MTLRegretBound}
Let Assumptions \ref{Ass:BoundedNorms}, \ref{Ass:ArmsDistribution}, \ref{Ass:Sparsity}, \ref{Ass:subGaussnoise} be satisfied. {Assume that $\max_{t\in [T]}\norm{\w_t}_2{\leq}L$, for some constant $L>0$ and that $N\geq N_0(\delta)$ for some $\delta\in (0,1)$}. Then, with probability at least $1-\delta$, the multi-task regret of Algorithm \ref{Alg:ASNucRegBan} is upper bounded ({up to logarithmic factors}) by
\begin{align*}
    \bigO&\bigg( (T\sqrt{rN}+\sqrt{rdTN})\bigg).
\end{align*}
\end{theorem}

Our proof is inspired by the one proposed in \citep{oh2020sparsity} for the single task lasso bandit approach. We present here a sketch summarizing the key steps. Full technical details are provided in Appendix \ref{AppSec:RegretUBound}. 
%\sout{A key step to bound the regret is to control the multi-task theoretical matrix $\Sigmabar$ with the multi-task adapted one $\Sigmabar_n$. This allows matrix $\Sigmabarhat_n$ to satisfy the RSC Condition, hence to met rates as in (\ref{Eq:Lemma1Ineq}).}

%\MP{if space permit we can extend the proof a bit}
\begin{proof}[Proof Sketch of Theorem \ref{Th:MTLRegretBound}.]
%In order to prove the regret bound we first introduce the following technical lemmas. 
We consider the instantaneous multi-task regret at round $n$
\begin{align*}
\Rbar_n &= \sum_{t=1}^T R_{t,n} = \sum_{t=1}^T  \langle \x_{t,n}^*,  \w_t \rangle  - \langle \x_{t,n} , \w_t \rangle.
\end{align*}
During the first $N_0(\delta)$ rounds, the RSC condition may not be satisfied. During this phase, using a simple conditioning argument, we obtain the following bound. We have with probability at least $1-\delta$, for any $n \in [N_0]$
\begin{align*}
\sum_{n=1}^{N_0}\Rbar_n &\lesssim T \, L\,N_0\,\sqrt{\log(eTN_0K\delta^{-1})}.
\end{align*}
%\textcolor{red}{Do we really need $K$ in the above bound?}
Starting from the $N_0+1$ round, we bound the instantaneous regret as follows: 
\begin{align*}
    \Rbar_{n} 
    &\leq  \sum_{n=N_0+1}^{N}\sum_{t=1}^T  \langle \x_{t,n}^* - \x_{t,n},  \w_t - \what_{t,n} \rangle.
\end{align*}
Using another conditioning, we get with probability at least $1-\delta$, for any $t\in [T]$, $n\in [N]$,
\begin{align*}
    \Rbar_{n} 
    &\lesssim  \sqrt{T} \norm{\W - \What_{n} }_F \sqrt{\log(eTNK\delta^{-1})}.
\end{align*}
This bound is of interest if we can guarantee that $\What_{n} $ is an accurate estimate of $\W$, meaning that $\norm{\W - \What_{n} }_F$ is small. 

To this end, we first prove that $\Sigmabarhat_n$ satisfies the RSC condition whp starting from round $N_0$. The argument is based on concentration bounds for martingales {combined with a novel geometric analysis of the cone of matrix $\mathcal{C}_r$; see Lemma \ref{Le:MatrixConcentration}.} 

\begin{proposition}\label{Le:FirstN0}
Let Assumptions \ref{Ass:BoundedNorms} and \ref{Ass:ArmsDistribution} be satisfied. Assume in addition that $\Sigmabar$  satisfies the RSC condition. Then, for any $\delta\in (0,1)$ and any $n\geq N_0(\delta)$, with probability at least $1-\delta$, the multi-task empirical matrix $\Sigmabarhat_n$ satisfies the RSC condition with constant $$
\kappa\big(\Sigmabarhat_n\big)\geq\frac{\kappa\big(\Sigmabar\big)}{4 \nu \omega_\Xcal}>0.
$$
\end{proposition}
Consequently, this means that we can now use Lemma \ref{Le:OracleInequality} and Proposition \ref{prop:subgaussian2} which guarantee with probability at least $1-\delta$, simultaneously for any $n\in [N_0,N]$ 
\begin{align*}
&\norm{ \What_{n} - \W}_{\rm F} \lesssim\sqrt{\frac{r}{n}} \biggl(  \sqrt{T+d+\log\left(\frac{4N(d+T)}{\delta}\right)} \\
&\hspace{3cm}\cdot\log^{3/2}\left(\frac{8N(d+T)}{\delta}\right)\biggr).
\end{align*}
Summing over $n$, we obtain with probability at least $1-\delta$ up to logarithmic factors
\begin{align*}
 \sum_{n=N_0+1}^{N}\Rbar_s &\leq C \sqrt{rT(T+d)N} ,
\end{align*}
where $C = C\left(L,\eta,\sigma,C_{\mathbf{z}},\kappa(\Sigmabar),\max_{1\leq k \leq K }\left\lbrace \opnorm{\Sigma_{k}}^{1/2}\right\rbrace\right)$ is a finite constant under our assumptions.

An union bound summing the regrets for the first phase $n\leq N_0$ and the second phase $n>N_0$ gives the result (up to a rescaling of the constants).

\end{proof}

\subsection{Result Discussion}

We now discuss the implication of Theorem \ref{Th:MTLRegretBound} and compare it to previous approaches to multi-task representation learning in the bandit setting.

{\bf Advantage over ITL.} Notice that running \textbf{any} $T$ independent policies (ITL) with the linear contextual bandit setting defined in Section \ref{SubSec:LinearBandits} %\karim{\sout{under Assumptions \ref{Ass:BoundedNorms}, \ref{Ass:ArmsDistribution}}} 
would yield at best a regret bound of order {$T \sqrt{dN}$, up to logs; see e.g. \citep[Chapter 19.4, comment 5]{lattimore2020bandit}}. %This can be shown considering the lower bound argument in \citep[See Chapter 24.1]{lattimore2020bandit}. 
%We would still have a lower bound of order $\bigO(Td\sqrt{N})$ if considering $\norm{\x}_2 \leq 1$ and assuming $d\leq 2N$ \citep[See Chapter 24.2]{lattimore2020bandit}. 
Since this regret bound is always larger than the upper bound in Theorem \ref{Th:MTLRegretBound} for the proposed MTL strategy, there is a gain in using our method. In particular, if $T>d$, discarding logarithmic factors, the bound for our method is smaller by a factor of order $O(\sqrt{d/r})$, while for $d > T$ the gain is of order $O(\sqrt{T/r})$.

{\bf Minimax Optimality}. Theorem 2 in \citep{yang2020impact} provides a matching minimax lower bound to our Theorem \ref{Th:MTLRegretBound} (up to logarithmic terms). In particular this implies that in the regime $T\geq d$, our policy achieves the minimax regret $T\sqrt{rN}$ (up to logs). This corresponds to the performance of the oracle policy which knows the true underlying representation a-priori.

{\bf Comparison to SOTA Approaches.} Our upper bound compares favorably to Theorem 1 in \citep{yang2020impact} for the finitely many arms setting. Their result assumes Gaussian arms with non singular covariance matrices.  %\sout{$\Sigma_k$ with $\lambda_{\max}(\Sigma_k) \le O(1/d)$ and $\lambda_{\min}(\Sigma_k) \ge \Omega(1 / d)$}.
They also assume that $\mathrm{rank}(\W)=r$ with \textbf{known} $r$ and $K, T \le \mathrm{poly}(d)$, $N \ge d^2$. Then Theorem 1 in \citep{yang2020impact} guarantees for their MLinGreedy policy that
\begin{align*}
    \E[\overline{R}(T,N)]  = &O\Big( \left(  T\sqrt{rN} + \sqrt{rdTN} \right)\cdot\\
    &\cdot\sqrt{\log(NKT)\log(NTdr)}\log \log N\Big). 
\end{align*}
%Note that their regret bound contains the logarithmic term $\sqrt{\log(NKT)\log(NTdr)}\log\log N$; see the end of the proof of their Theorem 1 on page 15.
We stress out that their policy requires the knowledge of the rank $r$ whereas our policy does not. Notably, their regret bound requires $N\geq d^2$ rounds to be valid whereas our regret bound is valid as soon as {$N\gtrsim r^2$ (up to logs)}. Moreover, their analysis requires the invertibility of the arms covariance, whereas we only need the less restrictive RSC condition. Finally we also extend the result to sub-Gaussian arm distributions. {With our notation, \cite{hu2021near} obtained a regret bound of order $O(T\sqrt{drN} + d\sqrt{rTN})$ up to logs for their MTLR-OFUL policy in the infinitely many arms setting,  provided the rank $r$ is known to their policy.}

\section{EXPERIMENTS}\label{Sec:Experiments}
In this section we validate the policy proposed in Section \ref{Sec:TransferLearningLinBan}. The experiments displayed in Figures \ref{Fig:Exp1} and \ref{Fig:Exp2} compare 3 different policies: the Trace-Norm Bandit approach of Algorithm \ref{Alg:ASNucRegBan} for different choices of the regularization parameter,
the Oracle Policy which knows the low dimensional representation $\Bbf\in\reals^{d\times r}$ to
select the arm to play at each round 
and the ITL policy which solves each $d$-dimensional task separately.
In order to compute the trace-norm estimator (Eq. \ref{Eq:TraceNormReg}) we adopt the accelerated gradient method proposed in \citep{ji2009accelerated}. 
%We represent on the $y$-axis the cumulative regret per task $\overline{R}(T,n)/T$ as a function of $n$.
%\karim{**We may need to modify the following sentence**}\textcolor{magenta}{Our policy parameter has been accurately validated over a logarithmic scale.} Finally, we adopted the cumulative reward gathered over all tasks as performance metric.
In all the experiments, we report results averaged over $5$ repetitions.
\begin{figure}[t!]
    \includegraphics[width=.45\textwidth]{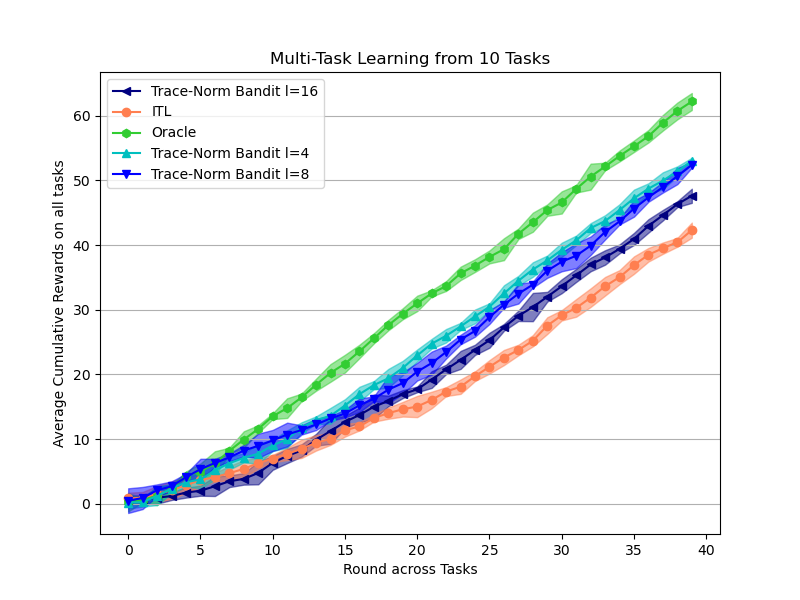}\\
    \includegraphics[width=.45\textwidth]{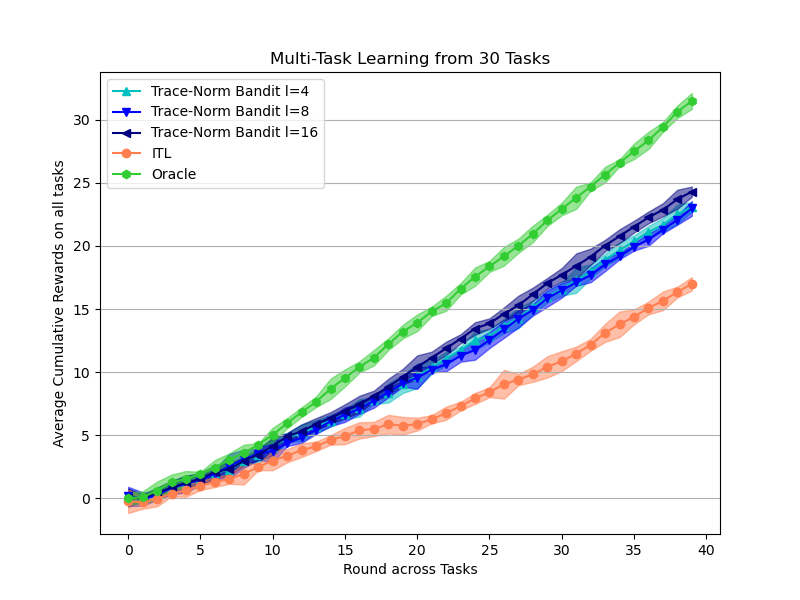}
    \caption{Averaged cumulative reward over all tasks for $T=10$ (top) and $T=30$ (bottom). Each task lasts for $N=40$ rounds, has $K=10$ arms with $d=20$ features, noise variance $\sigma^2=1$.}
    \label{Fig:Exp1}
\end{figure}
\begin{figure}[t!]
    \centering
    \includegraphics[width=.45\textwidth]{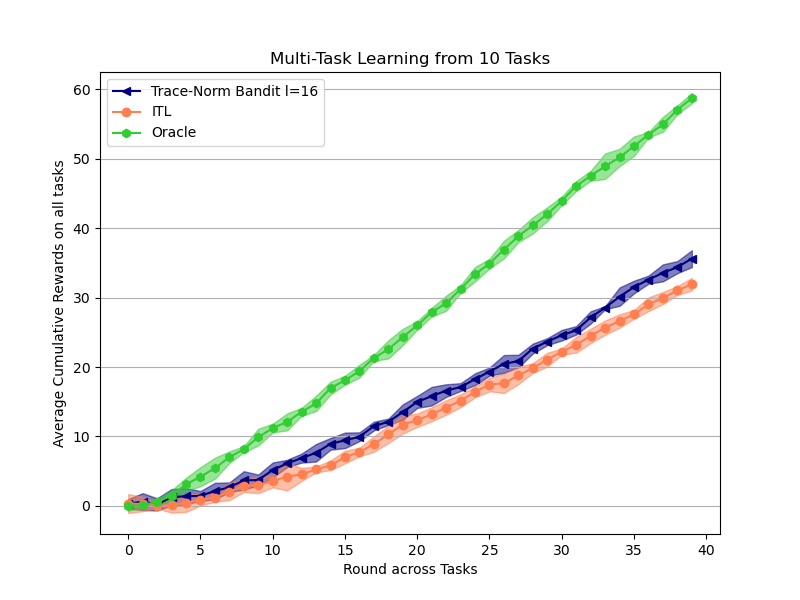}\\
    \includegraphics[width=.45\textwidth]{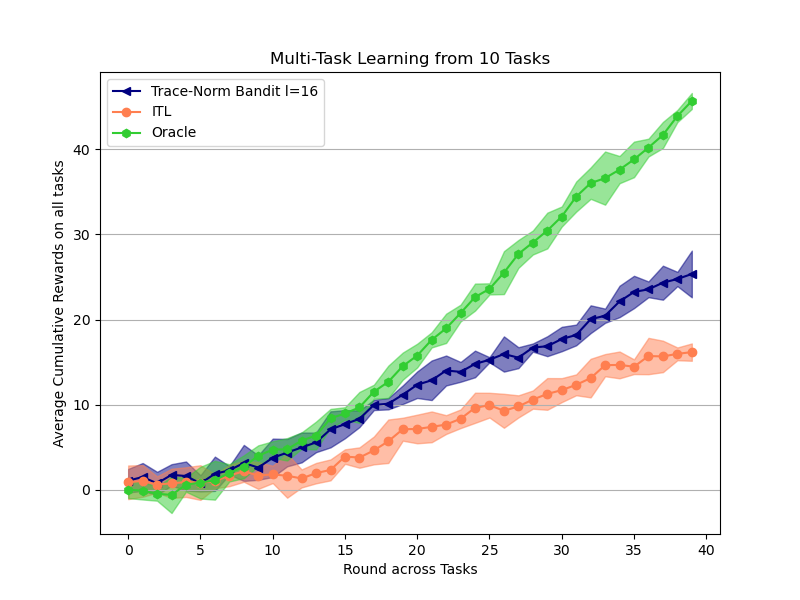}
    \caption{Averaged cumulative reward over all tasks. Each task lasts for $N=40$ rounds, has $K=10$ arms with $d=50$ features and noise variance $\sigma^2 =1$ (top), $\sigma^2=9$ (bottom).}
    \label{Fig:Exp2}
\end{figure}
%\paragraph

\noindent {\bf Data Processing.} We conduct numerical experiments on different configurations. Specifically, we analyze the impact of the number of tasks $T$, the dimension $d$, noise variance $\sigma^2$ and the choice of the regularization parameter. As metric we consider the cumulative reward averaged on all the tasks $\overline{R}(T,N)/T$ as a function of the number of rounds $N$. The arm set 
is generated randomly from a standard Gaussian distribution 
. The task matrix $\W$ has been chosen to be of rank $r\ll d$ with randomly generated gaussian entries. 

{\bf Result Discussion.} {Figures \ref{Fig:Exp1} and \ref{Fig:Exp2} indicate that the proposed multi-task approach performs favorably over independent task learning. Figure \ref{Fig:Exp1} highlights that our trace-norm policy performs well and significantly better than ITL policy as long as the regularization parameter used in Eq. \eqref{Eq:TraceNormReg} is taken large enough according to theory:
$$\lambda_n = l \Bigg[ \left(  \frac{T+d}{n} + \frac{\log\frac{2}{\delta}}{n}\right) \bigvee \Bigg( \sqrt{\frac{T{+}d}{n}} +  \sqrt{ \frac{ \log\frac{2}{\delta}}{n}}\Bigg)\Bigg].
$$
In addition, we also observe that the performance of the trace norm bandit improves and tends to that of the oracle policy as the number of tasks increases. 
In Figure \ref{Fig:Exp2}, we investigate the impact of noise variance on the performance of the different policies. We used a higher value of dimension $d=50$ as it is  well-known that the impact of the noise becomes more problematic in high-dimension. We observe that the  trace-norm policy is significantly less impacted by increased noise variance than the ITL policy. Indeed the performance of the ITL policy degrades by $48\%$ as the noise variance increases from $1$ to $9$. Comparatively, the oracle policy and the trace-norm policy only incur a degradation of about $25\%$ as the noise variance increases. This may be due to the fact that the oracle works in a subspace of dimension $r$, and that the trace-norm policy is able to perform dimensionality reduction through the nuclear norm regularization whereas the ITL policy works in the whole $d$-dimensional space.}

\section{LIMITATIONS AND FUTURE WORK}
We have studied the benefit of multi-task representation learning in the setting of linear contextual bandit tasks. We proposed a novel bandit policy based on trace-norm regularization which is computationally efficient and does not need knowledge of the rank of the underlying task matrix. We derived an upper bound for the multi-task regret of the proposed policy, showing that it is effective in comparison to learning the tasks independently. Additionally, this regret bound is minimax optimal and, in the regime $T\geq d$, our policy's regret matches the one of the oracle policy which knows the low-rank common representation.

%Our analysis we build a novel argument specific for the considered bandit setting. Finally, we evaluate and confirm the benefit of the multi-task approach both theoretically and empirically.
In this work we have restricted our analysis to the case of linear feature learning. Additionally, we still require the designed policy to know one problem parameter which is the variance associated to the noisy term. Relying on \citep{maurer2016benefit}, in the future an interesting extension would be to go beyond the linearity of the shared representation. Secondly, inspired by \citep{belloni2011square} a challenge one may try to solve is to design a fully parameter free policy. Namely, a policy which does not require to know the noise variance $\sigma^2$. %Finally, we are considering removing Assumption \ref{Ass:ArmsDistribution} starting from the Lasso bandit approach proposed in \citep{oh2020sparsity}.
Finally an important future direction is to consider the meta-learning setting, in which the tasks are observed sequentially and the goal is to minimize the regret on future yet-unseen tasks.

\subsubsection*{Acknowledgements}
This work was supported in part by the Chaire Business Analytic for Future Banking, PNRR MUR project PE0000013-FAIR, and the European Union (Project
101070617).
% Bibliograph
\bibliographystyle{plainnat}
\bibliography{main.bib}

\begin{thebibliography}{48}
\providecommand{\natexlab}[1]{#1}
\providecommand{\url}[1]{\texttt{#1}}
\expandafter\ifx\csname urlstyle\endcsname\relax
  \providecommand{\doi}[1]{doi: #1}\else
  \providecommand{\doi}{doi: \begingroup \urlstyle{rm}\Url}\fi

\bibitem[Abbasi-Yadkori et~al.(2011)Abbasi-Yadkori, P{\'a}l, and
  Szepesv{\'a}ri]{abbasi2011improved}
Yasin Abbasi-Yadkori, D{\'a}vid P{\'a}l, and Csaba Szepesv{\'a}ri.
\newblock Improved algorithms for linear stochastic bandits.
\newblock In \emph{NIPS}, volume~11, pages 2312--2320, 2011.

\bibitem[Abbasi-Yadkori et~al.(2012)Abbasi-Yadkori, Pal, and
  Szepesvari]{abbasi2012online}
Yasin Abbasi-Yadkori, David Pal, and Csaba Szepesvari.
\newblock Online-to-confidence-set conversions and application to sparse
  stochastic bandits.
\newblock In \emph{Artificial Intelligence and Statistics}, pages 1--9, 2012.

\bibitem[Argyriou et~al.(2008)Argyriou, Evgeniou, and
  Pontil]{argyriou2008convex}
Andreas Argyriou, Theodoros Evgeniou, and Massimiliano Pontil.
\newblock Convex multi-task feature learning.
\newblock \emph{Machine learning}, 73\penalty0 (3):\penalty0 243--272, 2008.

\bibitem[Ariu et~al.(2020)Ariu, Abe, and Prouti{\`e}re]{ariu2020thresholded}
Kaito Ariu, Kenshi Abe, and Alexandre Prouti{\`e}re.
\newblock Thresholded lasso bandit.
\newblock \emph{arXiv preprint arXiv:2010.11994}, 2020.

\bibitem[Auer(2002)]{auer2002using}
Peter Auer.
\newblock Using confidence bounds for exploitation-exploration trade-offs.
\newblock \emph{Journal of Machine Learning Research}, 3\penalty0
  (Nov):\penalty0 397--422, 2002.

\bibitem[Awerbuch and Kleinberg(2008)]{awerbuch2008online}
Baruch Awerbuch and Robert Kleinberg.
\newblock Online linear optimization and adaptive routing.
\newblock \emph{Journal of Computer and System Sciences}, 74\penalty0
  (1):\penalty0 97--114, 2008.

\bibitem[Bastani and Bayati(2020)]{bastani2020online}
Hamsa Bastani and Mohsen Bayati.
\newblock Online decision making with high-dimensional covariates.
\newblock \emph{Operations Research}, 68\penalty0 (1):\penalty0 276--294, 2020.

\bibitem[Bastani et~al.(2021)Bastani, Bayati, and Khosravi]{bastani2021mostly}
Hamsa Bastani, Mohsen Bayati, and Khashayar Khosravi.
\newblock Mostly exploration-free algorithms for contextual bandits.
\newblock \emph{Management Science}, 67\penalty0 (3):\penalty0 1329--1349,
  2021.

\bibitem[Basu et~al.(2021)Basu, Kveton, Zaheer, and Szepesv{\'a}ri]{basu2021no}
Soumya Basu, Branislav Kveton, Manzil Zaheer, and Csaba Szepesv{\'a}ri.
\newblock No regrets for learning the prior in bandits.
\newblock \emph{arXiv preprint arXiv:2107.06196}, 2021.

\bibitem[Baxter(2000)]{baxter2000model}
Jonathan Baxter.
\newblock A model of inductive bias learning.
\newblock \emph{Journal of artificial intelligence research}, 12:\penalty0
  149--198, 2000.

\bibitem[Belloni et~al.(2011)Belloni, Chernozhukov, and
  Wang]{belloni2011square}
Alexandre Belloni, Victor Chernozhukov, and Lie Wang.
\newblock Square-root lasso: pivotal recovery of sparse signals via conic
  programming.
\newblock \emph{Biometrika}, 98\penalty0 (4):\penalty0 791--806, 2011.

\bibitem[B{\"u}hlmann and Van De~Geer(2011)]{buhlmann2011statistics}
Peter B{\"u}hlmann and Sara Van De~Geer.
\newblock \emph{Statistics for high-dimensional data: methods, theory and
  applications}.
\newblock Springer Science \& Business Media, 2011.

\bibitem[Calandriello et~al.(2014)Calandriello, Lazaric, and
  Restelli]{calandriello2014sparse}
Daniele Calandriello, Alessandro Lazaric, and Marcello Restelli.
\newblock Sparse multi-task reinforcement learning.
\newblock In \emph{NIPS}, 2014.

\bibitem[Cella and Pontil(2021)]{cellaSparseMTL}
Leonardo Cella and Massimiliano Pontil.
\newblock Multi-task and meta-learning with sparse linear bandits.
\newblock In \emph{The Conference on Uncertainty in Artificial Intelligence},
  2021.

\bibitem[Cella et~al.(2020)Cella, Lazaric, and Pontil]{cella2020}
Leonardo Cella, Alessandro Lazaric, and Massimiliano Pontil.
\newblock Meta-learning with stochastic linear bandits.
\newblock In Hal~Daumé III and Aarti Singh, editors, \emph{Proceedings of the
  37th International Conference on Machine Learning}, volume 119 of
  \emph{Proceedings of Machine Learning Research}, pages 1360--1370. PMLR,
  13--18 Jul 2020.

\bibitem[Chu et~al.(2011)Chu, Li, Reyzin, and Schapire]{chu2011contextual}
Wei Chu, Lihong Li, Lev Reyzin, and Robert Schapire.
\newblock Contextual bandits with linear payoff functions.
\newblock In \emph{International Conference on Artificial Intelligence and
  Statistics}, pages 208--214, 2011.

\bibitem[Denevi et~al.(2018)Denevi, Ciliberto, Stamos, and
  Pontil]{denevi2018incremental}
Giulia Denevi, Carlo Ciliberto, Dimitris Stamos, and Massimiliano Pontil.
\newblock Incremental learning-to-learn with statistical guarantees.
\newblock In \emph{Proceedings of the Thirty-Fourth Conference on Uncertainty
  in Artificial Intelligence}, pages 457--466, 2018.

\bibitem[Denevi et~al.(2019)Denevi, Stamos, Ciliberto, and
  Pontil]{denevi2019online}
Giulia Denevi, Dimitris Stamos, Carlo Ciliberto, and Massimiliano Pontil.
\newblock Online-within-online meta-learning.
\newblock \emph{Advances in Neural Information Processing Systems}, 32, 2019.

\bibitem[D'Eramo et~al.(2019)D'Eramo, Tateo, Bonarini, Restelli, and
  Peters]{d2019sharing}
Carlo D'Eramo, Davide Tateo, Andrea Bonarini, Marcello Restelli, and Jan
  Peters.
\newblock Sharing knowledge in multi-task deep reinforcement learning.
\newblock In \emph{International Conference on Learning Representations}, 2019.

\bibitem[Filippi et~al.(2010)Filippi, Cappe, Garivier, and
  Szepesv{\'a}ri]{filippi2010parametric}
Sarah Filippi, Olivier Cappe, Aur{\'e}lien Garivier, and Csaba Szepesv{\'a}ri.
\newblock Parametric bandits: The generalized linear case.
\newblock In \emph{NIPS}, volume~23, pages 586--594, 2010.

\bibitem[Foster and Rakhlin(2020)]{foster2020beyond}
Dylan Foster and Alexander Rakhlin.
\newblock Beyond ucb: Optimal and efficient contextual bandits with regression
  oracles.
\newblock In \emph{International Conference on Machine Learning}, pages
  3199--3210. PMLR, 2020.

\bibitem[Hao et~al.(2020)Hao, Lattimore, and Wang]{SparseBanditsLowerBound}
Botao Hao, Tor Lattimore, and Mengdi Wang.
\newblock High-dimensional sparse linear bandits.
\newblock In H.~Larochelle, M.~Ranzato, R.~Hadsell, M.~F. Balcan, and H.~Lin,
  editors, \emph{Advances in Neural Information Processing Systems}, volume~33,
  pages 10753--10763. Curran Associates, Inc., 2020.

\bibitem[Hu et~al.(2021)Hu, Chen, Jin, Li, and Wang]{hu2021near}
Jiachen Hu, Xiaoyu Chen, Chi Jin, Lihong Li, and Liwei Wang.
\newblock Near-optimal representation learning for linear bandits and linear
  rl.
\newblock In \emph{International Conference on Machine Learning}, pages
  4349--4358. PMLR, 2021.

\bibitem[Ji and Ye(2009)]{ji2009accelerated}
Shuiwang Ji and Jieping Ye.
\newblock An accelerated gradient method for trace norm minimization.
\newblock In \emph{International Conference on Machine Learning}, pages
  457--464, 2009.

\bibitem[Kim and Paik(2019)]{kim2019doubly}
Gi-Soo Kim and Myunghee~Cho Paik.
\newblock Doubly-robust lasso bandit.
\newblock \emph{Advances in Neural Information Processing Systems},
  32:\penalty0 5877--5887, 2019.

\bibitem[Koltchinskii(2011)]{koltchinskii2011oracle}
V.~Koltchinskii.
\newblock \emph{Oracle Inequalities in Empirical Risk Minimization and Sparse
  Recovery Problems: {\'E}cole D’{\'E}t{\'e} de Probabilit{\'e}s de
  Saint-Flour XXXVIII-2008}.
\newblock Lecture Notes in Mathematics. Springer, 2011.
\newblock ISBN 9783642221460.
\newblock URL \url{https://books.google.fr/books?id=D5Jxen3\_xkAC}.

\bibitem[Koltchinskii et~al.(2011)Koltchinskii, Lounici, and
  Tsybakov]{koltchinskii2011nuclear}
Vladimir Koltchinskii, Karim Lounici, and Alexandre~B Tsybakov.
\newblock Nuclear-norm penalization and optimal rates for noisy low-rank matrix
  completion.
\newblock \emph{The Annals of Statistics}, 39\penalty0 (5):\penalty0
  2302--2329, 2011.

\bibitem[Kuzborskij et~al.(2019)Kuzborskij, Cella, and
  Cesa-Bianchi]{kuzborskij19a}
Ilja Kuzborskij, Leonardo Cella, and Nicol{\`o} Cesa-Bianchi.
\newblock Efficient linear bandits through matrix sketching.
\newblock In \emph{The 22nd International Conference on Artificial Intelligence
  and Statistics}, pages 177--185. PMLR, 2019.

\bibitem[Kveton et~al.(2017)Kveton, Szepesv{\'a}ri, Rao, Wen, Abbasi-Yadkori,
  and Muthukrishnan]{kveton2017low-rank}
Branislav Kveton, Csaba Szepesv{\'a}ri, Anup Rao, Zheng Wen, Yasin
  Abbasi-Yadkori, and S~Muthukrishnan.
\newblock Stochastic low-rank bandits.
\newblock \emph{arXiv preprint arXiv:1712.04644}, 2017.

\bibitem[Kveton et~al.(2021)Kveton, Konobeev, Zaheer, Hsu, Mladenov, Boutilier,
  and Szepesvari]{kveton2021meta}
Branislav Kveton, Mikhail Konobeev, Manzil Zaheer, Chih-wei Hsu, Martin
  Mladenov, Craig Boutilier, and Csaba Szepesvari.
\newblock Meta-thompson sampling.
\newblock \emph{arXiv preprint arXiv:2102.06129}, 2021.

\bibitem[Lattimore and Szepesv{\'a}ri(2020)]{lattimore2020bandit}
Tor Lattimore and Csaba Szepesv{\'a}ri.
\newblock \emph{Bandit algorithms}.
\newblock Cambridge University Press, 2020.

\bibitem[Li et~al.(2010)Li, Chu, Langford, and Schapire]{li2010contextual}
Lihong Li, Wei Chu, John Langford, and Robert~E Schapire.
\newblock A contextual-bandit approach to personalized news article
  recommendation.
\newblock In \emph{Proceedings of the 19th international conference on World
  wide web}, pages 661--670, 2010.

\bibitem[Li et~al.(2017)Li, Lu, and Zhou]{li2017provably}
Lihong Li, Yu~Lu, and Dengyong Zhou.
\newblock Provably optimal algorithms for generalized linear contextual
  bandits.
\newblock In \emph{International Conference on Machine Learning}, pages
  2071--2080, 2017.

\bibitem[Lounici et~al.(2011)Lounici, Pontil, Van De~Geer, and
  Tsybakov]{lounici2011oracle}
Karim Lounici, Massimiliano Pontil, Sara Van De~Geer, and Alexandre~B Tsybakov.
\newblock Oracle inequalities and optimal inference under group sparsity.
\newblock \emph{The annals of statistics}, 39\penalty0 (4):\penalty0
  2164--2204, 2011.

\bibitem[Lu et~al.(2021)Lu, Meisami, and Tewari]{lowRank_GLM}
Yangyi Lu, Amirhossein Meisami, and Ambuj Tewari.
\newblock Low-rank generalized linear bandit problems.
\newblock In Arindam Banerjee and Kenji Fukumizu, editors, \emph{Proceedings of
  The 24th International Conference on Artificial Intelligence and Statistics},
  volume 130 of \emph{Proceedings of Machine Learning Research}, pages
  460--468. PMLR, 13--15 Apr 2021.

\bibitem[Maurer and Pontil(2013)]{pontil2013excess}
Andreas Maurer and Massimiliano Pontil.
\newblock Excess risk bounds for multitask learning with trace norm
  regularization.
\newblock In \emph{Conference on Learning Theory}, pages 55--76. PMLR, 2013.

\bibitem[Maurer et~al.(2016)Maurer, Pontil, and
  Romera-Paredes]{maurer2016benefit}
Andreas Maurer, Massimiliano Pontil, and Bernardino Romera-Paredes.
\newblock The benefit of multitask representation learning.
\newblock \emph{Journal of Machine Learning Research}, 17\penalty0
  (81):\penalty0 1--32, 2016.

\bibitem[Negahban and Wainwright(2011)]{negahban2011estimation}
Sahand Negahban and Martin~J. Wainwright.
\newblock Estimation of (near) low-rank matrices with noise and
  high-dimensional scaling.
\newblock \emph{The Annals of Statistics}, 39(2):\penalty0 1069--1097, 2011.

\bibitem[Oh et~al.(2021)Oh, Iyengar, and Zeevi]{oh2020sparsity}
Min-hwan Oh, Garud Iyengar, and Assaf Zeevi.
\newblock Sparsity-agnostic lasso bandit.
\newblock In \emph{Proceedings of the 38-th International Conference on Machine
  Learning}, volume 139, pages 8271--8280, 2021.

\bibitem[Rusmevichientong and Tsitsiklis(2010)]{tsitsiklis}
Paat Rusmevichientong and John~N. Tsitsiklis.
\newblock Linearly parameterized bandits.
\newblock \emph{Mathematics of Operations Research}, page 395–411, 2010.

\bibitem[Shen et~al.(2015)Shen, Wang, Jiang, and Zha]{shen2015portfolio}
Weiwei Shen, Jun Wang, Yu-Gang Jiang, and Hongyuan Zha.
\newblock Portfolio choices with orthogonal bandit learning.
\newblock In \emph{Twenty-fourth international joint conference on artificial
  intelligence}, 2015.

\bibitem[Simchowitz et~al.(2021)Simchowitz, Tosh, Krishnamurthy, Hsu, Lykouris,
  Dud{\'\i}k, and Schapire]{simchowitz2021bayesian}
Max Simchowitz, Christopher Tosh, Akshay Krishnamurthy, Daniel Hsu, Thodoris
  Lykouris, Miroslav Dud{\'\i}k, and Robert~E Schapire.
\newblock Bayesian decision-making under misspecified priors with applications
  to meta-learning.
\newblock \emph{arXiv preprint arXiv:2107.01509}, 2021.

\bibitem[Tripuraneni et~al.(2021)Tripuraneni, Jin, and
  Jordan]{tripuraneni2021provable}
Nilesh Tripuraneni, Chi Jin, and Michael Jordan.
\newblock Provable meta-learning of linear representations.
\newblock In \emph{International Conference on Machine Learning}, pages
  10434--10443. PMLR, 2021.

\bibitem[Tropp(2011)]{tropp2011freedman}
Joel Tropp.
\newblock {Freedman's inequality for matrix martingales}.
\newblock \emph{Electronic Communications in Probability}, 16\penalty0
  (none):\penalty0 262 -- 270, 2011.
\newblock \doi{10.1214/ECP.v16-1624}.
\newblock URL \url{https://doi.org/10.1214/ECP.v16-1624}.

\bibitem[Vershynin(2018)]{vershynin2019}
Roman Vershynin.
\newblock \emph{High-Dimensional Probability: An Introduction with Applications
  in Data Science}.
\newblock Cambridge Series in Statistical and Probabilistic Mathematics.
  Cambridge University Press, 2018.
\newblock \doi{10.1017/9781108231596}.

\bibitem[Wang et~al.(2018)Wang, Wei, and Yao]{wang2018minimax}
Xue Wang, Mingcheng Wei, and Tao Yao.
\newblock Minimax concave penalized multi-armed bandit model with
  high-dimensional covariates.
\newblock In \emph{International Conference on Machine Learning}, pages
  5200--5208, 2018.

\bibitem[Woodroofe(1979)]{woodroofe1979one}
Michael Woodroofe.
\newblock A one-armed bandit problem with a concomitant variable.
\newblock \emph{Journal of the American Statistical Association}, 74\penalty0
  (368):\penalty0 799--806, 1979.

\bibitem[Yang et~al.(2020)Yang, Hu, Lee, and Du]{yang2020impact}
Jiaqi Yang, Wei Hu, Jason~D Lee, and Simon~Shaolei Du.
\newblock Impact of representation learning in linear bandits.
\newblock In \emph{International Conference on Learning Representations}, 2020.

\end{thebibliography}

\onecolumn
\newpage
%\vspace{.75truecm}
%\renewcommand{\thesubsection}{\Alph{subsection}}
\noindent{\bf \LARGE APPENDIX}

\appendix
\vspace{.5truecm}

\noindent This appendix provides full proofs of the results stated in the main body of the paper. It is organized as follows:
\begin{itemize}
%    \item Appendix \ref{AppSec:DefNot} we list the main notations and definitions used throughout the paper.
    \item Appendix \ref{AppSec:Lemma1} contains the proof of Lemma \ref{Le:OracleInequality}, which is the oracle inequality associated to the error $\Delta_{n+1} = \What_{n+1} - \W$ considering non i.i.d. data.
    \item Appendix \ref{AppSec:ControlDn} presents the proof of Proposition \ref{prop:subgaussian2} on  the control to the operator norm of matrix $\Dbf_n = \sum_{t=1}^T \sum_{i=1}^n \eta_{t,i}\, \x_{t,i} \otimes \e_{t}$.
    \item Appendix \ref{AppSec:Prop2} contains the proof of Proposition \ref{Le:FirstN0}. It relies on a matrix perturbation argument and a novel analysis of the uniform deviation of the empirical arms covariance on the cone $\mathcal{C}_r$ of approximately low-rank matrices % on the perator norm concentration inequalities on the deviations of the empirical arms covariance 
     (Lemmas \ref{Le:MatrixConcentration} and \ref{Le:RandomRSC}). These results are then combined to relate the RSC constant $\kappa(\Sigmabar)$ of $\Sigmabar$ to that of %associated to the theoretical covariance $\Sigmabf\in\reals^{d\times d}$ to 
    its empirical counterpart $\kappa(\Sigmabarhat_{n})$.
    %\ref{AppSec:Lemma2} and \ref{AppSec:Lemma3} contains the proofs of Lemmas \ref{Le:MatrixConcentration} and \ref{Le:RandomRSC}, respectively. These results allow us to connect the RSC constant $\kappa(\Sigmabar)$ associated to the theoretical covariance $\Sigma\in\reals^{d\times d}$ to its empirical counterpart $\kappa(\Sigmabarhat_{n})$.
    \item Appendix \ref{AppSec:RegretUBound} contains the proof for the regret upper bound associated to the policy of Alg. \ref{Alg:ASNucRegBan}.
    % \item In Appendix \ref{AppSec:RegretLBound} we prove the multi-task regret lower bound. Similarly, in Appendix \ref{AppSec:Remark} we give the proof of the regret lower bound relative to the high-dimensional single-task setting \citep{oh2020sparsity}.
    \item In Appendix \ref{AppSec:Exp}, we provide additional numerical experiments.
    %provide the specific measures associated to the plots displayed in the experimental section.
\end{itemize}

%The following table summarizes the main notation used throughout the paper.

\begin{table}[b!]
    \centering
    \renewcommand{\arraystretch}{1.2}
    \begin{tabular}{| c | c |}
        \hline Symbol & Description \\
        \hline \hline
                $[n]$ & The set $\{1,\dots,n\}$, given a  positive integer $n$\\\hline

        $T$ & Number of tasks  \\
        \hline    
        $N$ & Time horizon associated to each single task\\
        \hline    
        $d$ & Dimension of context vectors\\
        \hline
        $a\vee b$ & The maximum between $a$ and $b$ ($\max(a,b)$)\\
        \hline
        $\|\cdot\|_{\psi_2}$ & The sub-Gaussian norm with $\psi_2(s) = e^{s^2}-1$\\
        &(See e.g. page 215 in \cite{koltchinskii2011oracle})
        \\
        \hline
        {$\e_1,\dots,e_T \in\reals^T$} & The standard basis indicator vectors, i.e. $\e_{t,j} = 1$, if $j=t$\\
        &and $0$ otherwise, for all $t,j \in [T]$
        \\
        \hline
        $\W = [\w_1, \dots, \w_T] \in \reals^{d\times T}$ & Matrix of $T$ regression tasks\\
        &(we also use the notation $[\W]_t \equiv w_t,~t \in [T]$)\\
        \hline            
        $r$ & Rank of the task matrix $\W$\\
        \hline
        $K$ & Number of arms\\
        \hline
$p$ & Joint distribution on $\mathbb{R}^{dK}$ (from which $K$ arm vectors are sample) 
        %from which the $N$ decision sets $\Dcal_1,\dots,\Dcal_N$ are sampled i.i.d. 
\\ \hline
$\Dcal_{t,n}$,~~$t \in [T]$,~$n\in [N]$ & Decision sets (each containing $K$ arm vectors) sampled i.i.d. from $p$\\
\hline
        $\x_{t,n}\in\Dcal_{t,n}$ & Arm vector chosen in task $t\in[T]$ at round $n\in[N]$\\
        \hline
        $\x^*_{t,n}\in\Dcal_{t,n}$ & Optimal arm vector in task $t\in[T]$ during round $n\in[N]$\\
        \hline
        $\Sigmabf\in \reals^{d\times d}$ & Theoretical covariance matrix - see eq.~\eqref{eq:S-111}\\
        \hline
        $\Sigmabf_{t,n}\in \reals^{d\times d}$ & Adapted covariance matrix for task $t$ at round $n$; see eq.~\eqref{eq:S-333}\\
        \hline 
        $\Sigmahat_{t,n}\in \reals^{d\times d}$ & Empirical covariance matrix for task $t$ at round $n$; see eq.~\eqref{eq:S-222}\\ \hline
        $\Sigmabar\in \reals^{dT\times dT}$ & $T$-block diagonal matrix ${\rm diag}(\Sigmabf,\dots,\Sigmabf)$ \\
        \hline
        $\Sigmabar_n \in \reals^{dT\times dT}$ & $T$-block diagonal matrix ${\rm diag}(\Sigmabf_{1,n},\dots,\Sigmabf_{T,n})$ \\
        \hline
        $ \Sigmabarhat_n \in \reals^{dT\times dT}$ & $T$-block diagonal matrix ${\rm diag}(\Sigmahat_{1,n},\dots,\Sigmahat_{T,n})$\\
        \hline
        $\norm{\x},~\norm{\x}_1,~\norm{\x}_{\infty}$ & Euclidean, $\ell_1$ and maximum norm associated to a vector $\x$\\
        \hline
        $[\A]_t$ & The $t$-th column of matrix $\A \in \reals^{d\times T}$\\
        \hline 
        $\lambdamin(\Abf), \lambdamax(\Abf)$ & Minimum and maximum eigenvalues of a square symmetric matrix $\Abf$\\
        \hline 
        $\sigma_{\min}(\Abf), \sigma_{\max}(\Abf)$ & Minimum and maximum singular values of matrix $\Abf$\\
        \hline
        $\nucnorm{\Abf}$ & Trace norm of matrix $\Abf$ (sum of its singular values)\\
        \hline
        $\norm{\Abf}_{\rm F}$ & Frobenius norm of matrix $\Abf$ ($\ell_2$ norm o matrix elements / singular values)\\
        \hline
        $\opnorm{\Abf}$ & Operator norm of matrix $\Abf$ (maximum singular value)\\
        \hline
    \end{tabular}
  % \caption{Main notation used throughout the paper.}
    \label{Tab:Notation}
\end{table}

\section{PROOF OF LEMMA \ref{Le:OracleInequality}}\label{AppSec:Lemma1}
By definition of $\What_n$, we have for any $\W\in\mathbb{R}^{d\times T}$,
\begin{equation*}
    \frac{1}{n} \sum_{t=1}^T \norm{\y_{t,n} - \X_{t,n} [\What_{n+1}]_t }_2^2 + \lambda_n \nucnorm{\What_{n+1}} \leq \frac{1}{n} \sum_{t=1}^T \norm{\y_{t,n} - \X_{t,n} [\W]_t }_2^2 + \lambda_n \nucnorm{\W}.
\end{equation*}
We define the error matrix $\Delta_{n+1} = \What_{n+1} - \W$ and introduce the following operator ${\cal A}: \mathbb{R}^{d \times T} \rightarrow \mathbb{R}^{n \times T}$ and its adjoint ${\cal A}^*:  \mathbb{R}^{n \times T} \rightarrow \mathbb{R}^{d \times T}$ as
\begin{align*}
    [{\cal A}(\W)]^t_n &= \langle \x_{t,n} \e^\top_t , \W\rangle = \text{Tr}(\W\e_t \x_{t,n}) =
    %MASSI I remove these steps
    %\x_{t,n}^\top \W \e_t = \x_{t,n}^\top [\W]_t = 
    \x_{t,n}^\top \w_t\\
{\cal A}^*(\mathbf{H}_{n}) &= \sum_{t=1}^T\sum_{i=1}^n \x_{t,i} \e^\top_{t} \eta_{t,i} = \Dbf_n \hspace{1em} \in \reals^{d\times T},
\end{align*}
where $\etab_{t,n} = (\eta_{t,1},\ldots,\eta_{t,n})^\top \in \reals^n$ and 
$
\mathbf{H}_{n}:= \left(\etab_{1,n},\ldots,  \etab_{T,n}\right)
\in \mathbb{R}^{n\times T}.
$

Using this notation the following hold
\begin{align*}
    \frac{1}{n}\norm{{\cal A}(\Delta_{n+1})}_{\rm F}^2 &= \frac{1}{n}\sum_{t=1}^T \norm{\left[{\cal A}(\Delta_{n+1})\right]^t}_2^2 \leq \frac{1}{n} \langle \mathbf{H}_{n}, {\cal A}(\Delta_{n+1}) \rangle + \lambda_n \left( \nucnorm{\W} - \nucnorm{\What_{n+1}} \right)\\
    &\leq \frac{1}{n} \langle \mathbf{H}_{n}, {\cal A}(\Delta_{n+1}) \rangle + \lambda_n \left( \nucnorm{\What_{n+1} + \Delta_{n+1}} - \nucnorm{\What_{n+1}} \right)\\
    &\leq \frac{1}{n} \langle \mathbf{H}_{n}, {\cal A}(\Delta_{n+1}) \rangle + \lambda_n \left( \nucnorm{\Delta_{n+1}} \right).
\end{align*}
Considering now the first term on the RHS and applying Holder's inequality we have
\[
    \frac{1}{n}\lvert\langle \mathbf{H}_{n}, {\cal A}(\Delta_{n+1})\rangle\rvert = \frac{1}{n}\lvert\langle{\cal A}^*(\mathbf{H}_{n}), \Delta_{n+1}\rangle\rvert \leq \frac{1}{n} \opnorm{{\cal A}^*(\mathbf{H}_{n})} \nucnorm{\Delta_{n+1}}.
\]
Now, considering $\lambda_n\geq\frac{1}{n}\opnorm{{\cal A}^*(\mathbf{H}_{n})}$ the following holds:
\begin{equation}\label{Eq:IntermediateBefLemma1}
\frac{1}{n}\norm{{\cal A}(\Delta_{n+1})}_{\rm F}^2 \leq 2 \lambda_n \nucnorm{\Delta_{n+1}}.
\end{equation}

Relying on \citep[][Lemma 1]{negahban2011estimation} we can decompose the error matrix $\Delta_{n+1}$ as $\Delta_{n+1}'+\Delta_{n+1}''$ such that $\Delta_{n+1}'$ is of rank at most $2r$ and 
\begin{equation}\label{Eq:Lemma1Statement}
    \nucnorm{\Delta_{n+1}} \leq 4 \nucnorm{\Delta_{n+1}'} .
\end{equation}
Assuming the RSC condition to be met with constant $\kappa\left(\Sigmabarhat_n\right)$, starting from equation \eqref{Eq:IntermediateBefLemma1} we get
\begin{equation*}
    \norm{\Delta_{n+1}}_{\rm F}^2 \leq  = \frac{ \norm{{\cal A}(\Delta_{n+1})}_{\rm F}^2}{2n\kappa\big(\Sigmabarhat_n\big)}\leq \frac{\lambda_n\ \nucnorm{\Delta_{n+1}}}{\kappa\left(\Sigmabarhat_n\right)}.
\end{equation*}
These last two results combined with \citep[][Lemma 1]{negahban2011estimation} give that  $\nucnorm{\Delta'_{n+1}}\leq\sqrt{2r}\norm{\Delta'_{n+1}}_{\rm F}$, from which we conclude that
\begin{equation*}
    \norm{\Delta_{n+1}}_{\rm F} \leq \frac{32\lambda_n\sqrt{r}}{\kappa\left(\Sigmabarhat_n\right)}.
\end{equation*}

\section{PROOF OF PROPOSITION \ref{prop:subgaussian2}}
\label{AppSec:ControlDn}

\subsection{Preliminary results} 

We consider the stochastic process $\{\mathbf{M}_n\}_{n\geq 0}$ defined as $\mathbf{M}_0=0$ a.s. and for any $n\geq 1$
$$
\mathbf{M}_n = \sum_{t=1}^T \x_{t,n} \otimes \x_{t,n} - \mathbb{E}\left[\x_{t,n} \otimes \x_{t,n} \vert \Fbar_{n-1} \right].
$$
By definition of the Trace-Norm bandit (Algorithm \ref{Alg:ASNucRegBan}), $\{M_n\}_{n\geq 0}$ is a $\Fbar_{n-1}$-martingale. Furthermore, given the past history $\Fbar_{n-1}$, we select at round $n$ for each task $t$ the arm $\x_{t,n}\in \mathcal{D}_{t,n}$ where the sets $\mathcal{D}_{t,n}$, $t\in [T]$, are mutually independent. This means that the arms $\x_{t,n}$, $t\in [T]$, are mutually independent given the past history $\Fbar_{n-1}$.

\noindent Next, we use a standard argument to control the operator norm of $\mathbb{M}_n$. Fix $\epsilon\in (0,1/2)$. An $\epsilon$-net $\mathcal{N}_\epsilon$ of $\Sb^{d}$ is a subset of $\Sb^{d}$ such that for any $\ubf\in \Sb^{d}$, there exists $\vbf\in \mathcal{N}_\epsilon$ such that $\|\ubf-\vbf\|\leq \epsilon$. 
%Corollary 4.2.13 in \cite{vershynin2019} guarantees the existence of an $\epsilon$-net $\mathcal{N}_\epsilon\subset \mathcal{S}^{m}$ such that
% \begin{align*}
%   % \label{eq:Nepscard}
%     |\mathcal{N}_\epsilon| \leq \left(1 + \frac{2}{\epsilon} \right)^{m}.
% \end{align*}
%of cardinality $|\mathcal{N}_\epsilon| \leq \left(1 + \frac{2}{\epsilon} \right)^{d+T}$ such that
Corollary 4.2.13 in \cite{vershynin2019} guarantees the existence of an $\epsilon$-net $\mathcal{N}_\epsilon$ of $\Sb^{d}$ such that
\begin{align}
    \label{eq:Nepscardbis}
    |\mathcal{N}_\epsilon| \leq \left(1 + \frac{2}{\epsilon} \right)^{d}.
\end{align}
Similarly to the first step of the proof of Theorem 4.4.5 in \cite{vershynin2019}, we get
\begin{align}
\label{eq:unifboundquadbis}
  %\|\Dbf_n\|_{\infty}  =  
 \opnorm{ \mathbf{M}_n} \leq \frac{1}{1-2\epsilon} \max_{\bu\in \mathcal{N}_\epsilon}\left\lbrace \langle  \mathbf{M}_n \bu,\bu \rangle\right\rbrace\leq \frac{1}{1-2\epsilon}  \max_{\bu\in \mathcal{N}_\epsilon}\left\lbrace \sum_{t=1}^T \langle\x_{t,n} ,u\rangle^2  - \mathbb{E}[\langle\x_{t,n} ,u\rangle^2\vert \Fbar_{n-1}]\right\rbrace.
\end{align}
% Next we have
% \begin{align}
%    \opnorm{ \mathbf{M}_n} &= \max_{\bu\in \mathcal{N}_\epsilon}\left\lbrace \sum_{t=1}^T \langle\x_{t,n} ,\bu\rangle^2  - \mathbb{E}[\langle\x_{t,n} ,\bu\rangle^2\vert \Fbar_{n-1}]\right\rbrace.
% \end{align}
In view of Assumption \ref{Ass:BoundedNorms}, we have $\langle\x_{t,n} ,u\rangle = \langle  \mathbf{z}_{t,n} ,\Sigmabf_k^{1/2} \bu\rangle$ and $\mathbb{E}[\langle\x_{t,n} ,\bu\rangle^2\vert \Fbar_{n-1}] =\langle  \Sigmabf_{k(t)} \bu,\bu \rangle $ for some $k(t)\in [K]$. We apply now the Hanson-Wright's inequality conditionally on $\Fbar_{n-1}$ to get for any $x>0$
\begin{align*}
    &\mathbb{P}\left( \sum_{t=1}^{T} \langle  \mathbf{z}_{t,n} ,\Sigmabf_{k(t)}^{1/2} \bu\rangle^2 - \langle  \Sigmabf_{k(t)} \bu,\bu \rangle \geq  C \left(\sqrt{T \, \max_{k\in [K]}\{\norm{\Sigmabf_k^{1/2} \bu\otimes \bu \Sigmabf_k^{1/2}}_F^2\}  \, x} + \max_{k\in [K]}\{\opnorm{\Sigmabf_k \bu\otimes \bu \Sigmabf_k}\} \, x\right)\vert \Fbar_{n-1} \right) \leq e^{-x},
% \mathbf{z}_{t,n}^\top \Sigma_{k}\mathbf{z}_{t,n} \leq \mathbb{E}\left[  \mathbf{z}_{t,n}^\top \Sigma_{k}\mathbf{z}_{t,n}   \big \vert \Fbar_{n-1} \right]  + C\left(\sqrt{\|\Sigma_k\|_{\rm F}^2 x} + \opnorm{\Sigma_k}\ x\right)   \big\vert \Fbar_{n-1} \right) \geq 1 - e^{-x}.
\end{align*}
where $C>0$ is a numerical constant which can depend only on $C_{\mathbf{z}}$. Note that $\norm{\Sigmabf_k^{1/2} \bu\otimes \bu \Sigmabf_k^{1/2}}_F =\opnorm{\Sigmabf_k^{1/2} \bu\otimes \bu \Sigmabf_k^{1/2}} \norm{\bu}^2 \leq \opnorm{\Sigmabf_k}$. Hence we get for any $x>0$
\begin{align*}
    &\mathbb{P}\left( \sum_{t=1}^{T} \langle  \mathbf{z}_{t,n} ,\Sigmabf_{k(t)} \bu\rangle^2 - \langle  \Sigmabf_{k(t)} \bu,\bu \rangle \geq  C \max_{k\in [K]}\{\opnorm{\Sigmabf_k} \} \left(\sqrt{T \, x} +  x\right)\vert \Fbar_{n-1} \right) \leq e^{-x} .
% \mathbf{z}_{t,n}^\top \Sigma_{k}\mathbf{z}_{t,n} \leq \mathbb{E}\left[  \mathbf{z}_{t,n}^\top \Sigma_{k}\mathbf{z}_{t,n}   \big \vert \Fbar_{n-1} \right]  + C\left(\sqrt{\|\Sigma_k\|_{\rm F}^2 x} + \opnorm{\Sigma_k}\ x\right)   \big\vert \Fbar_{n-1} \right) \geq 1 - e^{-x}.
\end{align*}
We define now the event 
$$
\Omega_n = \bigcap_{\bu\in \mathcal{N}_\epsilon} \left\lbrace  \left|  \langle  \mathbf{M}_n \bv,\bu \rangle  \right| \leq C \max_{k\in [K]}\{\opnorm{\Sigmabf_k} \} \left(\sqrt{T \, x} +  x\right)  \right\rbrace.
$$
Set $\epsilon=1/4$. A simple union bound combining \eqref{eq:Nepscardbis} with the last two displays gives 
%and \eqref{eq:boundMn}
$$
\mathbb{P}\left(\Omega_n^c \right) \leq |\mathcal{N}_\epsilon|\, e^{-t'}\leq 2 e^{d\log(9) -x}.
$$
Now we set $x = \log(4\delta^{-1}) + d \log(9)$ for some $\delta \in (0,1)$. Consequently, we obtain that
$$
\mathbb{P}\left(\Omega_n \right)\geq 1-\delta/4.
$$
It follows, in view of \eqref{eq:unifboundquadbis}, with probability at least $1-\delta$
\begin{align}
\label{eq:Mnop}
\opnorm{\mathbf{M}_n}  \leq C'\max_{k\in [K]}\{\opnorm{\Sigmabf_k} \} \left(\sqrt{T \, (d+ \log(4\delta^{-1}) )} + d+ \log(4\delta^{-1})  \right),
%\left( c_{\eta} C_{\x} \,(\log(2\delta^{-1}) + (T+d) \log(9))  + \sigma\, \max_{1\leq k \leq K } \opnorm{\Sigma_{k}}^{1/2},\sqrt{ n \,  (\log(2\delta^{-1}) + (T+d) \log(9))} \right).
\end{align}
for some numerical constant $C'>0$ that can depend only on $C_{\mathbf{z}}$.

We note that
$$
\opnorm{\sum_{t=1}^T \mathbb{E}\left[\x_{t,n} \otimes \x_{t,n} \vert \Fbar_{n-1} \right]} \leq T \max_{k}\{ \opnorm{\Sigmabf_k} \},
$$
since $\sum_{t=1}^T \mathbb{E}\left[\x_{t,n} \otimes \x_{t,n} \vert \Fbar_{n-1} \right] = \sum_{t=1}^T \Sigmabf_{k(t)}$, where $k(t)\in [K]$. Combining this observation with \eqref{eq:Mnop}, we prove that the following event
%Define now the event
\begin{align}
\label{eq:omegan1}
    \Omega_n =  \left\lbrace \opnorm{\sum_{t=1}^T \x_{t,n} \otimes \x_{t,n}} \leq \max_{k\in [K]}\{\opnorm{\Sigmabf_k} \}\left( T + C' \left(\sqrt{T \, (d+ \log(4N\delta^{-1}) )} + d+ \log(4N\delta^{-1})  \right) \right) \right\rbrace
\end{align}
satisfies $\mathbb{P}(\Omega_n )\geq 1-\delta/(4N)$.
%for some large enough constant $C>0$ and for any $n\in [N]$
Next we also introduce the event
\begin{align}
\label{eq:omegan1bis}
\overline{\Omega}_n = \bigcap_{l=0}^n \Omega_l \in \Fbar_{n},
\end{align}
with $\Omega_0$ being the whole sample space. The Bayes rule 
%combined with \eqref{eq:Mnop} 
gives
\begin{align*}
%\label{eq:omegaNunionbound}
    \mathbb{P}\left(   \overline{\Omega}_N \right) = \prod_{n=1}^N    \mathbb{P}\left( {\Omega}_n \big\vert \bigcap_{k=0}^{n-1} {\Omega}_k  \right) \geq \left( 1-\frac{\delta}{4N} \right)^{N}.
\end{align*}
Bernoulli's inequality ($(1+x)^n\geq 1+nx$ for any $x>-1$ and integer $n\geq 1$) gives 
\begin{align}
\label{eq:OmegabarNbis}
    \mathbb{P}\left(   \overline{\Omega}_N \right) = \prod_{n=1}^N    \mathbb{P}\left({\Omega}_n \big\vert \bigcap_{k=0}^{n-1} {\Omega}_k  \right) \geq 1-\frac{\delta}{4}.
\end{align}

Define for any $n\in [N]$ the events
\begin{align}
\label{eq:omegan2} 
\Omega_n' = \bigcap_{t=1}^{T}\left\lbrace |\eta_{t,n}|\leq c_\eta  \sqrt{\log(8TN\delta^{-1})} \right\rbrace,
\end{align}
and
\begin{align}
\label{eq:omegan2bis} 
\overline{\Omega}_n' = \bigcap_{l=0}^n \Omega_l',
\end{align}
with $\Omega'_0$ being the whole sample space. 

Assumption \ref{Ass:subGaussnoise} guarantees that for any $\delta\in (0,1)$, with probability at least $1-\delta$
\begin{align}
%\label{eq:omeganinterm1}
    \mathbb{P}\left(   \Omega_n'     \big\vert \Fbar_{n-1}  \right) \geq 1-\frac{\delta}{4N}.
\end{align}
and
\begin{align}
\label{eq:omeganinterm1bis}
    \mathbb{P}\left(  \overline{\Omega}_N'     \big\vert \Fbar_{N-1}  \right) \geq 1-\frac{\delta}{4}.
\end{align}

\subsection{Main proof}

\paragraph{Checking the martingale structure.} Let us define the stochastic process $(\Dbf_n)_{n\geq 0}$ as $\Dbf_0=0$ a.s. and 
\begin{align}
    \Dbf_n:=\sum_{t=1}^T \sum_{i=1}^n  \eta_{t,i}\, \x_{t,i} \e^\top_{t} \in \mathbb{R}^{d\times T}.
\end{align}
By definition of the Trace-Norm bandit (Algorithm \ref{Alg:ASNucRegBan}), given the past history $\Fbar_{n-1}$, we select at round $n$ the arms for the T tasks: $\x_{t,n}\in \mathcal{D}_{t,n}$ with $\mathcal{D}_{t,n}$ independent of $(\eta_{t,n})_{1\leq t \leq T}$. This means that
$$
(\x_{t,n})_{1\leq t \leq T} \indep (\eta_{t,n})_{1\leq t \leq T} \bigg\vert \Fbar_{n-1}.
$$
%the sequence of actions $\{\x_{t,n}\}_{n=1}^N$ is predictable w.r.t. the filtration $\{\Fbar_{n}\}_{n\geq 0}$, i.e., for any $ n \geq 1$, $\x_{t,n}$ is $\Fbar_{n-1}$-measurable. 
Consequently, under Assumption \ref{Ass:BoundedNorms} and \ref{Ass:subGaussnoise}, $(\Dbf_n)_{n\geq 0}$ is a square-root integrable martingale adapted to the filtration $\{\Fbar_{n}\}_{n\geq 0}$.

We would like to apply the Freedman inequality for matrix martingales \citep[][Corollary 1.3 ]{tropp2011freedman}. However this result is for bounded martingales. Therefore it is not directly applicable to $(\Dbf_n)_{n\geq 0}$. To remedy this difficulty. We introduce the following stopping times:
\begin{align}
\tau_1 &= \inf\left\lbrace n\geq 0\,:\,  
%\min_{t\in [T]}
\opnorm{\sum_{t=1}^T \x_{t,n} \otimes \x_{t,n}} \geq \max_{k\in [K]}\{\opnorm{\Sigmabf_k} \}\left( T + C' \left(\sqrt{T \, (d+ \log(4N\delta^{-1}) )} + d+ \log(4N\delta^{-1})  \right) \right)
%\max_{k\in [K]}\left[ \mathrm{trace}(\Sigmabf_k)  + C\left(\sqrt{\|\Sigmabf_k\|_{\rm F}^2 (\log(4TK\ntrain\delta^{-1}))} + \opnorm{\Sigmabf_k}\ (\log(4TK\ntrain\delta^{-1})\right)\right]
\right\rbrace,
\end{align}
and
\begin{align}
\tau_2 = \inf\left\lbrace n\geq 0\,:\, \min_{t\in [T]} \left\lbrace  |\eta_{t,n}|  \right\rbrace \geq c_\eta \sigma \sqrt{\log(8TN\delta^{-1}))}\right\rbrace.
\end{align}
%for some large enough absolute constant $C>0$.

Again by definition of the Trace-Norm bandit  and assumptions on the noise, $\tau_1$ and $\tau_2$ are both stopping times relative to the filtration $\{\Fbar_{n}\}_{n\geq 0}$, so is $\tau=\tau_1\wedge \tau_2$. Hence the stopped process $\{\Dbf^{\tau}_n\}_{n\geq 0}$ defined as $\Dbf^{\tau}_n = \Dbf_{n\wedge \tau}$ is also a martingale adapted to the filtration $\{\Fbar_{n}\}_{n\geq 0}$. Furthermore $\{\Dbf^{\tau}_n\}_{n\geq 0}$ is a bounded martingale. Hence we can apply the Freedman inequality to it.

By definition of $\tau_1$ and \eqref{eq:omegan1bis}, we have $\{\tau_1 > N\}= \overline{\Omega}_N$. Similarly for $\tau_2$ in view of \eqref{eq:omegan2bis}, we have $\{\tau_2 > N\}= \overline{\Omega}'_N$. Hence
\begin{align}
\label{eq:intermbis2}
\mathbb{P}\left( \opnorm{\Dbf_n} \geq t  \right) &= \mathbb{P}\left( \{ \Dbf_n \geq t \} \cap \{\tau > N\}  \right) + \mathbb{P}\left( \{ \opnorm{\Dbf_n} \geq t \} \cap \{\tau \leq N\}  \right) \notag\\
&\leq \mathbb{P}\left(  \opnorm{\Dbf^{\tau}_n} \geq t  \right) + \mathbb{P}\left( \{\tau_1 \leq N  \} \cup \{\tau_2 \leq N  \} \right) \notag\\
&= \mathbb{P}\left(  \opnorm{\Dbf^{\tau}_n} \geq t  \right) + \mathbb{P}\left( \tau_1 \leq N  \right) +  \mathbb{P}\left(  \tau_2 \leq N  \right)\notag\\
&\leq \mathbb{P}\left(  \opnorm{\Dbf^{\tau}_n} \geq t  \right) +\delta/2,
\end{align}
where we have used \eqref{eq:OmegabarNbis} and \eqref{eq:omeganinterm1bis} in the last line.

\paragraph{Application of Freedman's inequality.} We now use \citep[][Corollary 1.3]{tropp2011freedman} to control $\{\Dbf^{\tau}_n\}_{n\geq 0}$.

%We will apply the Freedman inequality for matrix martingales.
\begin{theorem}[Corollary 1.3 in \cite{tropp2011freedman}]
\label{thm:freedman}
Consider a matrix martingale $\{ \bm{Y}_n : n = 0, 1, 2, \dots \}$ whose values are matrices with dimension $d_1 \times d_2$, and let $\{ \bm{X}_k : n = 1, 2, 3, \dots \}$ be the difference sequence.  Assume that the difference sequence is uniformly bounded:
$$
\opnorm{ \bm{X}_n } \leq R
\quad\text{almost surely}
\quad\text{for $n = 1, 2, 3, \dots$}.
$$
Define two predictable quadratic variation processes for this martingale:
\begin{align*}
\bm{W}_{{\rm col}, \, k} &:= \sum\nolimits_{j=1}^n  
\Expect_{j-1} \big(\bm{X}_j \bm{X}_j^\adj \big) \quad\text{and} \\
\bm{W}_{{\rm row}, \, k} &:= \sum\nolimits_{j=1}^n
\Expect_{j-1} \big(\bm{X}_j^\adj \bm{X}_j \big)
\quad\text{for $n = 1, 2, 3, \dots$}.
\end{align*}
Then, for all $t \geq 0$ and $\overline{\sigma}^2 > 0$,
$$
\Pro\left(\exists n \geq 0 : \norm{ \bm{Y}_n } \geq t  \text{ and }\ 
	\max\{ \opnorm{ \bm{W}_{{\rm col}, \, n} }, \opnorm{\bm{W}_{{\rm row}, \, n}} \} \leq \overline{\sigma}^2 \right)
	\leq (d_1 + d_2) \cdot \exp \left\{ - \frac{ -t^2/2 }{\overline{\sigma}^2 + Rt/3} \right\}.
$$
\end{theorem}

\medskip
\noindent
We now check the conditions of Theorem \ref{thm:freedman}. In view of \eqref{eq:omegan1}-\eqref{eq:omegan1bis} and \eqref{eq:omegan2}-\eqref{eq:omegan2bis}, we have on the event $\{\tau>N\}= \{\tau_1>N\}\cap \{\tau_2>N\}$, for any $n\in [N]$,
\begin{align*}
    \opnorm{\sum_{t=1}^T\eta_{t,n} \x_{t,n} \otimes \e_t }&= \sqrt{\opnorm{\left(\sum_{t=1}^T\eta_{t,n} \x_{t,n} \otimes \e_t \right) \left(\sum_{t=1}^T\eta_{t,n} \e_t \otimes \x_{t,n}\right) } } \\
    &= \sqrt{\opnorm{\sum_{t=1}^T\eta_{t,n}^2 \x_{t,n} \otimes \x_{t,n} } } \leq \max_{t\in [T]}\{|\eta_{t,n}|\}\sqrt{\opnorm{\sum_{t=1}^T \x_{t,n} \otimes \x_{t,n}} } \\
     &\leq c_{\eta}\sqrt{\log(8TN\delta^{-1})} \sqrt{\max_{k\in [K]}\{\opnorm{\Sigmabf_k} \}\left( T + C' \left(\sqrt{T \, (d+ \log(4N\delta^{-1}) )} + d+ \log(4N\delta^{-1})  \right) \right) }
     %\left( \sqrt{T} \bigvee \sqrt{ (\log(4\ntrain\delta^{-1}))}  \right) C^{1/2}_{TK\ntrain}(\delta)\\
    % &\leq C \left( \sqrt{T} \bigvee \sqrt{ \log(4\ntrain\delta^{-1})}  \right) \max_{k\in [K]}\left\lbrace \opnorm{\Sigmabf_k}^{1/2} \right\rbrace \left( \sqrt{d}+ \sqrt{ \log(4TK\ntrain \delta^{-1})}\right)  
    \\
    &\leq c_{\eta}\sqrt{1+C'/2}\sqrt{\log(8TN\delta^{-1})} \sqrt{\max_{k\in [K]}\{\opnorm{\Sigmabf_k} \}\left( T  + d+ \log(4N\delta^{-1})  \right) }.
\end{align*}

\noindent
Next we have
\begin{align*}
\bm{W}_{{\rm col}}&=\sum_{s=1}^n \mathbb{E}\left[   \sum_{t=1}^{T}\sum_{s=1}^{n} \eta_{t,s}^2 \x_{t,s} \otimes   \x_{t,s}  \vert \Fbar_{s-1} \right] = \sigma^2 \sum_{t=1}^{T}\sum_{s=1}^n \Sigmabf_{k(t,s)},
\end{align*}
where $k(t,s)\in [K]$ for any $t,s$. Hence
$$
\opnorm{\bm{W}_{{\rm col}}}\leq Tn\, \sigma^2 \max_{k\in [K]}\left\lbrace \opnorm{\Sigmabf_k}\right\rbrace.
$$
We proceed similarly for $\bm{W}_{{\rm row}}$
\begin{align*}
\bm{W}_{{\rm row}}&=\sum_{s=1}^n \mathbb{E}\left[   \sum_{t=1}^{T}\sum_{s=1}^{n} \eta_{t,s}^2 \norm{\x_{t,s}}^2 \e_t\otimes \e_t  \vert \Fbar_{s-1} \right] \leq n \sigma^2 \max_{k\in [K]}\left\lbrace \mathrm{tr}(\Sigmabf_k)\right\rbrace I_T.
\end{align*}

\medskip
\noindent
Hence, we get
\begin{align*}
\opnorm{\bm{W}_{{\rm row}}}& \leq n \sigma^2 \max_{k\in [K]}\left\lbrace \mathrm{tr}(\Sigma_k)\right\rbrace\leq nd\, \sigma^2  \, \max_{k\in [K]}\left\lbrace \opnorm{\Sigmabf_k}\right\rbrace .
\end{align*}

\medskip
\noindent Define the event

\begin{align*}
&A_{\ntrain} = 
\left\lbrace 
\frac{ \opnorm{\Dbf_{\ntrain}}  }
{\ntrain}\leq \max_{k\in [K]}\{\opnorm{\Sigmabf_k}^{1/2} \}
\left(\sigma \sqrt{ \frac{(d+T) \, (\log(2\delta^{-1}N(d+T)))}{\ntrain}
} 
\right.
\right.
\\
&
\left.\hspace{2cm} 
\bigvee  
\frac{ c_{\eta}\sqrt{1+C'/2}\sqrt{\log(8TN\delta^{-1})} \sqrt{\left( T  + d+ \log(4N\delta^{-1})  \right) } \,(\log(2\delta^{-1}N(d+T)))}
{\ntrain}   
\right\rbrace.
\end{align*}
%for some large enough constant $C = C(\sigma,c_{\eta},C_\mathbf{z},\max_{k\in [K]}\left\lbrace \opnorm{\Sigma_k}^{1/2} \right\rbrace  )>0$.

Applying Theorem \ref{thm:freedman}, we get for any $t>0$
\begin{align}
   \mathbb{P}\left(  A_{\ntrain}^c \vert \bar{\Omega}_{\ntrain}
   \right)\leq \frac{\delta}{2N}.
\end{align}
%for some constant $C = C(\sigma,\max_{k\in [K]}\left\lbrace \opnorm{\Sigma_k}^{1/2} \right\rbrace,  )>0$.
%=C(\sigma,c_{\eta})

\medskip
\noindent
An elementary argument combining the previous display with \eqref{eq:intermbis2} gives
$$
\mathbb{P}\left(   A_{\ntrain}^c \right) = \mathbb{P}\left(   A_{\ntrain}^c \cap \bar{\Omega}_{\ntrain} \right)  + \mathbb{P}\left(   A_{\ntrain}^c \cap \bar{\Omega}_{\ntrain}^c \right)\leq \mathbb{P}\left(   A_{\ntrain}^c \vert \bar{\Omega}_{\ntrain} \right)  + \mathbb{P}\left(   \bar{\Omega}_{\ntrain}^c \right) \leq \frac{\delta}{N}.
$$

\noindent From the previous display and an union bound, we immediately deduce, with probability at least $1-\delta$,
\begin{align}
\label{eq:stoproofthm1}
&\max_{1\leq n \leq N} \frac{\opnorm{\Dbf_n}}{n}  \leq  \max_{k\in [K]}\{\opnorm{\Sigmabf_k}^{1/2} \}
\left(\sigma \sqrt{ \frac{(d+T) \, (\log(2\delta^{-1}N(d+T)))}{\ntrain}
} 
\right.
\notag\\
&
\left.\hspace{2cm} 
\bigvee  
\frac{ c_{\eta}\sqrt{1+C'/2}\sqrt{\left( T  + d+ \log(4N\delta^{-1})  \right) } \,(\log^{3/2}(8\delta^{-1}N(d+T)))}
{\ntrain} \right) .
\end{align}

\section{PROOF OF PROPOSITION \ref{Le:FirstN0}}
\label{AppSec:Prop2}

\subsection{Concentration bounds on the arms covariance.
%Proof of Lemma \ref{Le:MatrixConcentration}
}

For every $t \in [T]$, we recall that the empirical covariance matrix for task $t$ as
\begin{equation}
\label{eq:S-222-bis}
\Sigmahat_{t,n}=\frac{1}{n}\sum_{i=1}^n\x_{t,i} \x_{t,i}^\top
\end{equation}
and the corresponding adapted covariance matrix (w.r.t. the filtration $(\Fbar_{n})_{n\geq 0}$) as
\begin{equation}
\label{eq:S-333}
    \Sigmabf_{t,n}{=}\frac{1}{n}\sum_{i=1}^n\E\left[\x_{t,i} \x_{t,i}^\top \Big|\Fbar_{i-1}\right].
\end{equation} 
Here \textit{adapted} means that the covariance matrix $\Sigmabf_{t,n}$ is fully known at time $n$ since it depends only on the past history $\Fbar_{n-1}$.

Moreover, we use the notation $\Sigmabar,\Sigmabarhat_n,\Sigmabar_n \in\reals^{dT\times dT}$ for the theoretical, the empirical and the adapted multi-task matrices, respectively. They are all block diagonal and composed by the corresponding $T$ single task $d\times d$ matrices on the diagonal (e.g. $\Sigmabar_n= {\rm diag}(\Sigmabar_{1,n},\dots,\Sigmabar_{T,n})$). 

The following lemma gives a high probability bound on the deviation of 
%shows an entry-wise uniform deviation inequality of 
$\Sigmabarhat_n$ from $\Sigmabar_n$ according to the operator norm restricted to the cone $ \mathcal{C}_r$ defined in \eqref{eq:coneCr}. We define for any symmetric $dT\times dT$ matrix $A$
$$
\norm{A}_{\rm{op},\mathcal{C}_r} = \max_{\Deltabf\in \mathcal{C}_r\,:\, \norm{\Deltabf}_F=1
} \left\lbrace \langle A\rm{Vec}(\Deltabf),\rm{Vec}(\Deltabf)\rangle \right\rbrace.
$$

\begin{lemma}[Concentration]\label{Le:MatrixConcentration}
Let Assumption \ref{Ass:BoundedNorms} be satisfied. For any $n\geq 1$ and $\delta\in (0,1)$, with probability at least $1- \delta$, for any $n\in [N]$
\begin{align}
\label{eq:opnormN0}
% \opnorm{\Sigmabarhat_n - \Sigmabar_{n}}
\norm{\Sigmabarhat_n - \Sigmabar_{n}}_{\rm{op},\mathcal{C}_r}
&\leq  C  \max_{1\leq k \leq K }\left\lbrace\opnorm{\Sigmabf_{k}}\right\rbrace  r  \left( \frac{(1+\log(4NT\delta^{-1}))\log^2(d)  }{\sqrt{n}} + \frac{(1+\log(4NT\delta^{-1}))^2\log^2(d)}{n}\right).
\end{align}
for some constant $C = C(C_{\mathbf{z}})>0$.
\end{lemma}
\begin{proof}

For any $n\in [N]$, we define the $dT\times dT$ block-diagonal matrix 
$$
\mathbf{M}_{n}:=\mathrm{diag}(\mathbf{M}_{1,n},\ldots,\mathbf{M}_{T,n}),
$$
where the diagonal $d\times d$ diagonal block matrices are defined as follows:
\begin{align}
 \mathbf{M}_{t,n}:=  n \left(\Sigmahat_{t,n} - \Sigmabf_{t,n}\right) =  \sum_{s=1}^n  \x_{t,s} \x^\top_{t,s} - \E\left[\x_{t,s} \x^\top_{t,s} \Big|\Fbar_{s-1}\right].
\end{align}
We also set $\mathbf{M}_{0} = \mathbf{0}_{dT\times dT}$ a.s.
By construction $(\mathbf{M}_{n})_{n\geq 0}$ is a $\{\Fbar_n\}_{n}$-martingale satisfying
%Since $(\mathbf{M}_{n})_{n\geq 0}$ admits a block-diagonal structure, we have
\begin{align}
\label{eq:MnopCr}
    \norm{\mathbf{M}_{n}}_{\rm{op},\mathcal{C}_r} = \max_{\Deltabf\in \mathcal{C}_r\,:\, \norm{\Deltabf}_F=1
} \left\lbrace \langle \mathbf{M}_{n}\rm{Vec}(\Deltabf),\rm{Vec}(\Deltabf)\rangle \right\rbrace = \max_{\Deltabf\in \mathcal{C}_r\,:\, \norm{\Deltabf}_F=1
}\left\lbrace  \sum_{t=1}^{T} \langle \mathbf{M}_{t,n}\Deltabf_t,\Deltabf_t\rangle\right\rbrace.
\end{align}
We derive several geometric properties of matrices in the cone $\mathcal{C}_r$. We then exploit these properties via an improved union bound argument to derive a sharper uniform deviation bound of the martingale $(\mathbf{M}_{n})_{n\geq 0}$ on the cone $\mathcal{C}_r$ as compared to the standard operator norm bound %obtained via an union bound 
on the whole space of $d\times T$ matrices.

\paragraph{Uniform deviation bound on $\langle \mathbf{M}_{t,n}\bu,\bv\rangle$ for vectors $\bu$ $\bv$ in lower-dimensional balls.} We consider two linear subspaces $\mathbf{U}$, $\mathbf{V}$ of $\reals^d$ of respective dimension $d_{\mathbf{U}}$ and $d_{\mathbf{V}}$. We then introduce the ball $B_{\mathbf{U}}(0,r_{\mathbf{U}})\in \mathbf{U}$ centered at $0$ of radius $r_{\mathbf{U}}$. We define similarly the ball $B_{\mathbf{V}}(0,r_{\mathbf{V}})\in \mathbf{V}$. Fix $\bu\in B_{\mathbf{U}}(0,r_{\mathbf{U}})$ and $\bv\in B_{\mathbf{V}}(0,r_{\mathbf{V}})$. We consider now the $\{\Fbar_n\}_{n}$-martingale
$
%\begin{align}
   ( \langle \mathbf{M}_{t,n}\bu,\bv\rangle)_{n\geq 0}
%\end{align}
$.
Note that
\begin{align}
\label{eq:unifboundquad}
  %\|\Dbf_n\|_{\infty}  =  
 \max_{\bu\in B_{\mathbf{U}}(0,r_{\mathbf{U}}), \bv \in B_{\mathbf{V}}(0,r_{\mathbf{V}})}{\langle \mathbf{M}_{t,n} \bu,\bv  \rangle }  \leq r_{\mathbf{U}}\, r_{\mathbf{V}}  \max_{\bu\in B_{\mathbf{U}}(0,1), \bv \in B_{\mathbf{V}}(0,1)}{\langle \mathbf{M}_{t,n} \bu,\bv  \rangle }   \leq r_{\mathbf{U}}\, r_{\mathbf{V}} \opnorm{P_{\mathbf{U}}\mathbf{M}_{t,n}   P_{\mathbf{V}} }.
\end{align}
Using Lemma \ref{lem:techlemma1} with $\delta$ replaced by $\delta/(2T)$. gives with probability at least $1-\delta$, for any $n\in [N]$ and any $t\in [T]$
\begin{align}
\label{eq:unifboundquadbis}
  %\|\Dbf_n\|_{\infty}  =  
 &\max_{\bu\in B_{\mathbf{U}}(0,r_{\mathbf{U}}), \bv \in B_{\mathbf{V}}(0,r_{\mathbf{V}})}{\left\langle \frac{\mathbf{M}_{t,n}}{n} \bu,\bv  \right\rangle }\notag\\
 &\hspace{1cm}\leq C r_{\mathbf{U}}\, r_{\mathbf{V}}  \max_{k\in [K]} \{\opnorm{\Sigmabf_{k}}\} \biggl(  \sqrt{\frac{\bigl(  d_{\mathbf{U}} \vee d_{\mathbf{V}}+ \log(4NT\delta^{-1}) \bigr)\log(8NT\delta^{-1}d) }{n}} \notag\\
    &\hspace{6cm} + \frac{\sqrt{\left( d_{\mathbf{U}} + \log(4NT\delta^{-1}) \right)\left( d_{\mathbf{V}} + \log(4NT\delta^{-1}d) \right) }\log(4NT\delta^{-1})}{n}  \biggr).
\end{align}

\paragraph{Geometric properties of $\mathcal{C}_r$.} 

Set $\bar{r} =\mathrm{rank}(\Deltabf)$. Note that $\bar{r} \leq d\wedge T$. Taking the SVD of $\Deltabf$, we have
$$
\Deltabf = \sum_{j=1}^{\bar{r}} \sigma_j(\Deltabf)\bu_j(\Deltabf)\otimes \bv_j(\Deltabf),
$$
with singular values $\sigma_1(\Deltabf)\geq \sigma_{2}(\Deltabf)\geq \cdots\geq \sigma_{\bar{r}}(\Deltabf) >0$ and orthonormal families $\{\bu_j(\Deltabf)\}_{j=1}^{\bar{r}}\in \mathbb{R}^d$,  $\{\bv_j(\Deltabf)\}_{j=1}^{\bar{r}}\in \mathbb{R}^T$. For the sake of brevity, we set $\sigma_j(\Deltabf)=\sigma_j$, $\bu_j(\Deltabf)=\bu_j$ and $\bv_j(\Deltabf)=\bv_j$ for any $j\in [\bar{r}]$. From the previous displays, we immediately get the following representation for the columns of $\Deltabf$:
\begin{align}
\label{eq:Deltanorm-prop3}
    \Deltabf_t = \sum_{j=1}^{\bar{r}} \sigma_j \bv_{j,t} \bu_j,\quad \forall t\in {T},
\end{align}
where $\bv_j$ admits components $\bv_j = (\bv_{j,1},\cdots,\bv_{j,T})^\top$.

By definition of the cone $\mathcal{C}_r$ in \eqref{eq:coneCr}, for any $\Deltabf\in \mathcal{C}_r$, we have $\norm{\Pi(\Deltabf)}_* \leq 3 \norm{\Deltabf-\Pi(\Deltabf)}_*$. Note that $\mathrm{rank}(\Deltabf-\Pi(\Deltabf))\leq 2r$ by definition of $\Pi$ and the cone $\mathcal{C}_r$. Hence we have, for any $\Deltabf\in \mathcal{C}_r$ with $\norm{\Deltabf}_F=1$,
\begin{align}
\label{eq:Deltanorm-prop1}
    \norm{\Deltabf}_* \leq \norm{\Pi(\Deltabf)}_* + \norm{\Deltabf-\Pi(\Deltabf)}_*\leq 4 \norm{\Deltabf-\Pi(\Deltabf)}_* \leq 4\sqrt{2r} \norm{\Deltabf-\Pi(\Deltabf)}_F \leq 4\sqrt{2r} \norm{\Deltabf}_F = 4\sqrt{2r}.
\end{align}
We deduce from \eqref{eq:Deltanorm-prop1} that $\sum_{j=1}^{\bar{r}}\sigma_j \leq 4\sqrt{2r}$ and consequently for any $\Deltabf\in \mathcal{C}_r$
\begin{align}
    \label{eq:Deltanorm-prop2}
    \sigma_j(\Deltabf)\leq \frac{4\sqrt{2r}}{j},\quad \forall j\in [\bar{r}].
\end{align}
We conclude this paragraph with some elementary facts on orthonormal basis. We complete the orthonormal family $\{\bv_j\}_{j=1}^{\bar{r}}$ into an orthonormal basis $\{\bv_j\}_{j=1}^{T}$ of $ \mathbb{R}^T$. By properties of orthonormal basis, we have $\langle \bv_j, \bv_{j'}\rangle = \sum_{t=1}^{T}\bv_{j,t} \bv_{j',t} = \delta_{j,j'}$ where $\delta_{j,j'} = 1$ if $j=j'$ and $0$ otherwise.

\paragraph{Uniform bound over $\Deltabf \in \mathcal{C}_r$.} Our goal is to control $\norm{\mathbf{M}_{n}}_{\rm{op},\mathcal{C}_{r}} $. To this end, we set $m_* = \lceil \log_2(d \wedge T)\rceil$ and, for any $m\in [m_*]$, define $J_m = \{j\in [m_*]\,:\, 2^{m-1}\leq j < 2^m\}$. Next we propose the following decomposition of $\Deltabf$:
$$
\Deltabf_t = \sum_{m=1}^{m_*}\Deltabf_{t,m},\quad \text{where}\quad \Deltabf_{t,m} = \sum_{j\in J_m} \sigma_j \bv_{j,t} \bu_j,\quad \forall t\in [T].
$$
By construction and \eqref{eq:Deltanorm-prop2}, for any $m\in [m_*]$, all the vectors $\Deltabf_{t,m}$, ${t\in [T]}$ live in the same subspace of dimension at most $2^{m-1}$ and
$$
\norm{\Deltabf_{t,m}}^2\leq \sigma^2_{2^{m-1}} \sum_{j\in J_m}  \bv^2_{j,t}  \leq \sigma^2_{2^{m-1}}\leq \frac{32 r}{2^{2(m-1)}},
$$
and
\begin{align}
\label{eq:Deltatmnorm}
    \sum_{t=1}^{T}\norm{\Deltabf_{t,m}}^2\leq \sigma^2_{2^{m-1}} \sum_{j\in J_m}  \sum_{t=1}^{T}\bv^2_{j,t}  \leq \sigma^2_{2^{m-1}} |J_m|\leq \frac{32 r}{2^{(m-1)}}.
\end{align}
Combining the last two displays with \eqref{eq:Deltanorm-prop3}, we get that
\begin{align}
 \sum_{t=1}^{T} \langle \mathbf{M}_{t,n}\Deltabf_t,\Deltabf_t\rangle &=  \sum_{t=1}^{T} \sum_{m,m'=1}^{m_*} \langle \mathbf{M}_{t,n}\Deltabf_{t,m},\Deltabf_{t,m'}\rangle.
\end{align}
An union bound combining the last two displays with \eqref{eq:MnopCr} and \eqref{eq:unifboundquadbis} %Lemma \ref{lem:techlemma1}
gives for any $\delta\in (0,1)$ with probability at least $1-\delta$, for any $n\in [N]$
\begin{align*}
\frac{1}{n}\norm{\mathbf{M}_{n}}_{\rm{op},\mathcal{C}_r} &\leq  C  \max_{1\leq k \leq K }\left\lbrace\opnorm{\Sigmabf_{k}}\right\rbrace  \sum_{m,m'=1}^{m_*} \sum_{t=1}^{T} \norm{\Deltabf_{t,m}} \norm{\Deltabf_{t,m'}}     \epsilon_{m, m'},
\end{align*}
where
\begin{align*}
    \epsilon_{m,m'} &=  \sqrt{\frac{\bigl(  2^{m \vee m'-1}+ \log(4NT\delta^{-1}) \bigr)\log(4NT\delta^{-1}d) }{n}} \\
    &\hspace{3cm}+ \frac{\sqrt{\left( 2^{m-1} + \log(4NT\delta^{-1}) \right)\left( 2^{m'-1} + \log(4NT\delta^{-1}d) \right) }\log(4NT\delta^{-1})}{n}.
\end{align*}

The Cauchy-Schwartz inequality and \eqref{eq:Deltatmnorm} give
\begin{align*}
\norm{\mathbf{M}_{n}}_{\rm{op},\mathcal{C}_r} &\leq  C  \max_{1\leq k \leq K }\left\lbrace\opnorm{\Sigmabf_{k}}\right\rbrace  \sum_{m,m'=1}^{m_*} \left(\sum_{t=1}^{T} \norm{\Deltabf_{t,m}}^2\right)^{1/2} \left(\sum_{t=1}^{T}\norm{\Deltabf_{t,m'}}^2\right)^{1/2} \epsilon_{m, m'}\\
&\leq  32 C   \max_{1\leq k \leq K }\left\lbrace\opnorm{\Sigmabf_{k}}\right\rbrace r \sum_{m,m'=1}^{m_*} \frac{1}{2^{\frac{m-1}{2}}}\frac{1}{2^{\frac{m'-1}{2}}}\epsilon_{m,m'}.
\end{align*}
By definition of $\epsilon_{m, m'}$, we have for any $m,m'\in [m_*]$
\begin{align*}
   \frac{1}{2^{\frac{m-1}{2}}}\frac{1}{2^{\frac{m'-1}{2}}}\epsilon_{m,m'}   &\lesssim  \frac{1+\log(4NT\delta^{-1})  }{\sqrt{n}} + \frac{(1+\log(4NT\delta^{-1}))^2}{n}.
\end{align*}
Combining the last two displays with the fact that $m_*\lesssim \log(d)$, we get for any $\delta\in (0,1)$ with probability at least $1-\delta$, for any $n\in [N]$
\begin{align*}
\frac{1}{n}\norm{\mathbf{M}_{n}}_{\rm{op},\mathcal{C}_r} &\leq  C  \max_{1\leq k \leq K }\left\lbrace\opnorm{\Sigmabf_{k}}\right\rbrace  r  \left( \frac{(1+\log(4NT\delta^{-1}))\log^2(d)  }{\sqrt{n}} + \frac{(1+\log(4NT\delta^{-1}))^2\log^2(d)}{n}\right).
\end{align*}

\end{proof}

\subsection{RSC condition and Matrix perturbation
%Proof of Lemma \ref{Le:RandomRSC}
}\label{AppSec:Lemma3}

The next lemma guarantees that if the %entry-wise uniform 
$\norm{\cdot}_{\mathrm{op},\mathcal{C}_r}$ metric between two matrices is small enough and if one of these two matrices satisfies the RSC condition, then the other matrix also satisfied the RSC condition.
%\noindent It remains to specify the proper value for constant $C$. To this end, we derive a new result concerning the RSC condition when considering random matrices. 
Specifically, the next lemma links the RSC constant $\kappa\big(\Sigmabar\big)$ associated to the multi-task theoretical matrix, to the adapted one $\kappa\big(\Sigmabar_n\big)$.
\begin{lemma}[RSC condition and Random Matrices]\label{Le:RandomRSC}
Let ${\bf \Sigma}_0$ and ${\bf \Sigma}_1$ be two $dT\times dT$ block-diagonal matrices (with blocks matrices of size $d\times d$). 
Suppose that the RSC condition is met for the covariance matrix ${\bf \Sigma}_0$ with constant $\kappa\big({\bf \Sigma}_0\big)$ and that $\norm{{\bf \Sigma}_0-{\bf \Sigma}_1}_{\mathrm{op},\mathcal{C}_r}\leq\lambdatilde$, where
    \[
        \lambdatilde \leq \kappa\big({\bf \Sigma}_0\big).
    \]
    Then, the RSC condition is also met for the covariance matrix ${\bf \Sigma}_1$ with constant $\kappa\big({\bf \Sigma}_1\big)\geq\kappa\big({\bf \Sigma}_0\big)/2$. Moreover, for all $\Deltabf\in\reals^{d\times T}$ such that $\Deltabf \in \Ccal(r)$ (see equation \ref{Eq:Cset}), we have
    \begin{equation}
        \norm{\rm{Vec}(\Deltabf)}_{{\bf \Sigma}_0}^2 \leq 2 \norm{\rm{Vec}(\Deltabf)}_{{\bf \Sigma}_1}^2 \leq 3 \norm{\rm{Vec}(\Deltabf)}_{{\bf \Sigma}_0}^2.
    \end{equation}
\end{lemma}

The proof combines Definition \ref{Def:RSC} and the $\norm{\cdot}_{\mathrm{op},\mathcal{C}_r}$ metric closeness $\norm{{\bf \Sigma}_0 - {\bf \Sigma}_1}_{\mathrm{op},\mathcal{C}_r}\leq \lambdatilde$.

\begin{proof}
For any $\Deltabf\in \mathcal{C}_r$, we have
\begin{align*}
    \left| \norm{\text{Vec}(\Deltabf)}_{{\bf \Sigma}_0}^2 - \norm{\text{Vec}(\Deltabf)}_{{\bf \Sigma}_1}^2 \right| &\leq \norm{{\bf \Sigma}_0{-}{\bf \Sigma}_1}_{\mathrm{op},\mathcal{C}_r} \norm{\text{Vec}(\Deltabf)}_2^2 \leq \lambdatilde \norm{\text{Vec}(\Deltabf)}_2^2.
\end{align*}
According to Definition \ref{Def:RSC}, for any $\Deltabf\Deltabf\in \mathcal{C}_r$, the RSC condition gives
\[
    \norm{\Deltabf}^2_{\rm F} = \norm{\text{Vec}(\Deltabf)}^2_2 \leq \frac{\norm{\text{Vec}(\Deltabf)}^2_{\Sigma_0}}{2\kappa\big(\Sigmabf_0\big)}.
\]
Combining these two last results we have
\[
    \left | \; \norm{\text{Vec}(\Deltabf)}_{\Sigmabf_0}^2 - \norm{\text{Vec}(\Deltabf)}_{\Sigmabf_1}^2 \right | \leq \lambdatilde \frac{\norm{\text{Vec}(\Deltabf)}^2_{\Sigmabf_0}}{2\kappa\big(\Sigmabf_0\big)}.
\]
Equivalently, the following result holds
\[
    \left| \frac{\norm{\text{Vec}(\Deltabf)}_{\Sigmabf_1}^2}{\norm{\text{Vec}(\Deltabf)}_{\Sigmabf_0}^2} - 1\right| \leq \frac{\lambdatilde}{2\kappa\big(\Sigmabf_0\big)}.
\]
The statement follows directly by the selected $\lambdatilde$.
\end{proof}

\subsection{Combining everything}

\noindent We can now prove Proposition \ref{Le:FirstN0}.

\begin{proof}

In the RSC condition (see Definition \ref{Def:RSC}) the considered covariance matrix $\Sigmabar \in\reals^{dT\times dT}$ is block-diagonal consisting of $T$ blocks, one for each different task.
Let us consider now the single-task adapted matrix
\begin{align*}
    \Sigmabf_{t,n} = \frac{1}{n}\sum_{s=1}^n \E\left[\x_{t,s} \x_{t,s}^\top \Big| \Fbar_{s-1}\right]
    %&\frac{1}{n}\sum_{s=1}^n\sum_{i=1}^K \E\left[\x_{t,s,i}\otimes \x_{t,s,i} \Ind\left\{\x_{s,n,i}=\arg\max_{\x\in\Dcal_{s,n}} \x^\top \what_{s,n} \right\} \Big| \Fbar_{s-1} \right],
\end{align*}
and the multi-task one $\Sigmabar_{n}={\rm diag}(\Sigmabf_{1,n},...,\Sigmabf_{T,n})\in\reals^{dT\times dT}$.
Relying on \citep[Lemma 10]{oh2020sparsity}, under Assumption \ref{Ass:ArmsDistribution} we have
\begin{equation}\label{Eq:UBProofLemma2}
    \Sigmabar_{n} = \frac{1}{n}\sum_{s=1}^n \E\left[ \xbar_s  \xbar_s^\top  \Big| \Fbar_{s-1} \right]  \succeq {\left(2 \nu \omega_\Xcal\right)}^{-1} \Sigmabar,
\end{equation}
where $\xbar_s = [\x_{1,s}^\top,\cdots,\x_{T,s}^\top]^\top\in\reals^{dT}$. Now, let us denote \(\wtilde_{n} = \arg\min_{\wbar\in \Ccal(r)}\frac{\mathrm{Vec}(\wbar)\Sigmabar_{n}\mathrm{Vec}(\wbar)}{2 \norm{\mathrm{Vec}(\wbar)}^2_2}\).

Relying on the RSC condition and thanks to the previous display, we obtain
\[
    \frac{\mathrm{Vec}(\wtilde)^\top_{n}\Sigmabar_{n}\mathrm{Vec}(\wtilde)_{n}}{\norm{\mathrm{Vec}(\wtilde)_{n}}_2^2} \geq \frac{\mathrm{Vec}(\wtilde)^\top_{n}\Sigmabar \mathrm{Vec}(\wtilde)_{n}}{2 \nu \omega_\Xcal \norm{\mathrm{Vec}(\wtilde)_{n}}_2^2} \geq \frac{\kappa\big(\Sigmabar\big)}{2 \nu \omega_\Xcal}.
\]
Hence, $\Sigmabar_{n}$ satisfies the RSC condition with constant \(
\kappa\left(\Sigmabar_{n}\right)= \kappa\left(\Sigmabar\right)
/ \left(2 \nu \omega_\Xcal\right)\).

Next we consider the operator norm deviation of the multi-task adapted matrix $\Sigmabf_0=\Sigmabar_{n}$ from the multi-task empirical matrix $\Sigmabf_1=\Sigmabarhat_{n}$. In view of \eqref{eq:opnormN0}, the quantity $\norm{\Sigmabarhat_{n}-\Sigmabar_{n}}_{\mathrm{op},\mathcal{C}_r}$ becomes smaller as $n$ increases. This means that if we take $n$ large enough, then we have whp that $\norm{\Sigmabarhat_{n}-\Sigmabar_{n}}_{\mathrm{op},\mathcal{C}_r}\leq  \kappa\left(\Sigmabar\right)
/ \left(4 \nu \omega_\Xcal\right) = \kappa(\Sigmabar_{n})/2$. We formalize this argument below.

For a fixed $\delta\in (0,1)$, recall the definition of $N_{0}(\delta)$ in \eqref{eq:N0_simplified}.
In view of Lemma \ref{Le:MatrixConcentration}, if $n\geq N_0(\delta)$ then with probability at least $1-\delta$, 
$$
\norm{\Sigmabarhat_n - \Sigmabar_{n}}_{\mathrm{op},\mathcal{C}_r}  \leq \frac{\kappa\big(\Sigmabar\big)}{4 \nu \omega_\Xcal} \leq \frac{\kappa\big(\Sigmabar_n\big)}{2}.
$$
Then Lemma \ref{Le:RandomRSC} yields the result. 

\end{proof}

\subsection{Technical Lemma}

Consider linear subspaces $\mathbf{U},\mathbf{V}$ of $\reals^{d}$ of dimension $d_{\mathbf{U}}$ and $d_{\mathbf{V}}$ respectively. Denote by $P_{\mathbf{U}}$ and $P_{\mathbf{V}}$ the orthogonal projections onto $\mathbf{U}$ and $\mathbf{V}$ respectively. We have for any $n\in [N]$
$$
\max_{\bu\in B_{\mathbf{U}}(0,1), \bv \in B_{\mathbf{V}}(0,1)}{\langle \mathbf{M}_{t,n} \bu,\bv  \rangle } \leq  \opnorm{P_{\mathbf{U}} \mathbf{M}_{t,n} P_{\mathbf{V}}}.
$$

\begin{lemma}
\label{lem:techlemma1}
Let Assumption \ref{Ass:BoundedNorms} be satisfied. For any $\delta\in (0,1)$ and any $t \in [T]$, with probability at least $1-\delta$,
\begin{align}
\max_{n\in [N]}\left\lbrace \frac{1}{n}\opnorm{P_{\mathbf{U}} \mathbf{M}_{t,n} P_{\mathbf{V}}} \right\rbrace &\leq C   \max_{k\in [K]} \{\opnorm{\Sigmabf_{k}}\} \biggl(  \sqrt{\frac{\bigl(  d_{\mathbf{U}} \vee d_{\mathbf{V}}+ \log(4N\delta^{-1}) \bigr)\log(4N\delta^{-1}d) }{n}} \notag\\
    &\hspace{1cm} + \frac{\sqrt{\left( d_{\mathbf{U}} + \log(4N\delta^{-1}) \right)\left( d_{\mathbf{V}} + \log(4N\delta^{-1}d) \right) }\log(4N\delta^{-1})}{n}  \biggr),
 \end{align}
 where $C=C(C_{\mathbf{z})})>0$ may depend only on $C_{\mathbf{z}}$.
\end{lemma}

\begin{proof}
We have for any $n\geq 1$
\begin{align}
\label{eq:proofopnorm0}
\opnorm{P_{\mathbf{U}} (\mathbf{M}_{t,n} -\mathbf{M}_{t,n-1} )P_{\mathbf{V}}} &= \opnorm{P_{\mathbf{U}}(\x_{t,n}) P_{\mathbf{V}}(\x_{t,n})^\top -  \E\left[P_{\mathbf{U}}(\x_{t,n}) P_{\mathbf{V}}(\x_{t,n})^\top \Big|\Fbar_{n-1}\right]}\notag\\
&\leq \norm{P_{\mathbf{U}}(\x_{t,n})}\norm{P_{\mathbf{V}}(\x_{t,n})} + \max_{k\in [K]}\{\opnorm{\Sigmabf_k}\}.
%\\
%&\leq b^2d,\quad a.s.
\end{align}
In view of Assumption \ref{Ass:BoundedNorms}, we have $\|P_{\mathbf{U}}(\x_{t,n})\|_2^2 = \mathbf{z}_{t,n}^\top P_{\mathbf{U}}\Sigmabf_{k}P_{\mathbf{U}}\mathbf{z}_{t,n}$ for some $k\in [K]$. We apply now Hanson-Wright's inequality conditionally on $\Fbar_{n-1}$ to get for any $x>0$
\begin{align*}
    &\mathbb{P}\left(   \mathbf{z}_{t,n}^\top P_{\mathbf{U}}\Sigmabf_{k}P_{\mathbf{U}}\mathbf{z}_{t,n} \leq \mathbb{E}\left[  \mathbf{z}_{t,n}^\top P_{\mathbf{U}}\Sigmabf_{k}P_{\mathbf{U}}\mathbf{z}_{t,n}   \big \vert \Fbar_{n-1} \right]  + C\, C_{\mathbf{z}}^2\left(\sqrt{\|P_{\mathbf{U}}\Sigmabf_k P_{\mathbf{U}}\|_{\rm F}^2 x} + \opnorm{P_{\mathbf{U}}\Sigmabf_k P_{\mathbf{U}}}\ x\right)   \big\vert \Fbar_{n-1} \right) \geq 1 - e^{-x}.
\end{align*}
Note that $\mathbb{E}\left[  \mathbf{z}_{t,n}^\top P_{\mathbf{U}}\Sigmabf_{k}P_{\mathbf{U}}\mathbf{z}_{t,n}   \big \vert \Fbar_{n-1} \right]  = \mathrm{trace}(P_{\mathbf{U}}\Sigmabf_k P_{\mathbf{U}})\leq \max_{k\in [K]}\{\opnorm{\Sigmabf_k}\} d_{\mathbf{U}}$ since $\mathbf{z}_k$ is a zero mean isotropic random vector. 

We define the following event in $\Fbar_n$
\begin{align}
\label{eq:OmeganterU}
\Omega_n(\mathbf{U}) &:= %\bigcap_{\x\in  \mathcal{D}_{t,n} } 
\left\lbrace \norm{P_{\mathbf{U}}(\x_{t,n})}^2 \leq 
\max_{k\in [K]} \left\lbrace \opnorm{\Sigmabf_{k} } \right\rbrace 
\left( d_{\mathbf{U}}  + C\, C_{\mathbf{z}}^2\left( \sqrt{  d_{\mathbf{U}} \, \log(4N\delta^{-1})} 
+  \log(4N\delta^{-1})
\right) \right)
\right\rbrace,
\end{align}
and
\begin{align}
\label{eq:OmeganterV}
\Omega_n(\mathbf{V}) &:= %\bigcap_{\x\in  \mathcal{D}_{t,n} }
\left\lbrace \norm{P_{\mathbf{V}}(\x_{t,n})}^2 \leq 
\max_{k\in [K]} \left\lbrace \opnorm{\Sigmabf_{k} } \right\rbrace 
\left( d_{\mathbf{V}}  + C\,C_{\mathbf{z}}^2\left( \sqrt{  d_{\mathbf{V}} \, \log(4N\delta^{-1})} 
+  \log(4N\delta^{-1})
\right) \right)
\right\rbrace.
\end{align}
Combining the last three displays and taking $x=\log(\delta^{-1})$ immediately implies for any $n\in [N]$
\begin{align}
\label{eq:omeganinterm1ter}
    \mathbb{P}\left(   \Omega_n(\mathbf{U})     \big\vert \Fbar_{n-1}  \right) \geq 1-\frac{\delta}{4N},\quad \mathbb{P}\left(   \Omega_n(\mathbf{V})     \big\vert \Fbar_{n-1}  \right) \geq 1-\frac{\delta}{4N}
\end{align}

We define, for any $n\in [N]$
\begin{align}
\label{eq:eventOmegabarn}
\bar{\Omega}_n = \bigcap_{l=0}^n \Omega_l(\mathbf{U}) \cap \Omega_l(\mathbf{V})\in \Fbar_{n},
\end{align}
with $\Omega_0(\mathbf{U}) = \Omega_0(\mathbf{V})$ being the whole sample space. The Bayes rule combined with \eqref{eq:omeganinterm1ter} gives
\begin{align}
\label{eq:omegaNunionbound}
    \mathbb{P}\left(   \bar{\Omega}_N \right) = \prod_{n=1}^N    \mathbb{P}\left( \Omega_n(\mathbf{U}) \cap \Omega_n(\mathbf{V}) \big\vert \bigcap_{k=0}^{n-1} \Omega_k(\mathbf{U}) \cap \Omega_k(\mathbf{V})  \right) \geq \left( 1-\frac{\delta}{2N} \right)^{N}.
\end{align}
Bernoulli's inequality ($(1+x)^n\geq 1+nx$ for any $x>-1$ and integer $n\geq 1$) gives 
\begin{align}
\label{eq:OmegabarNter}
    \mathbb{P}\left(   \bar{\Omega}_N \right) = \prod_{n=1}^N    \mathbb{P}\left( \Omega_n(\mathbf{U}) \cap \Omega_n(\mathbf{V})\big\vert \bigcap_{k=0}^{n-1}\Omega_k(\mathbf{U}) \cap \Omega_k(\mathbf{V})  \right) \geq 1-\frac{\delta}{2}.
\end{align}
Define the stopping times $\tau(\mathbf{U})$ and $\tau(\mathbf{V})$ as follows
\begin{align}
\tau(\mathbf{U}) &= \inf\left\lbrace n\geq 0\,:\,  
%\min_{t\in [T]}
\norm{P_{\mathbf{U}}(\x_{t,n})}^2 \geq 
\max_{k\in [K]} \left\lbrace \opnorm{\Sigmabf_{k} } \right\rbrace 
\left( d_{\mathbf{U}}  + C\, C_{\mathbf{z}}^2\left( \sqrt{  d_{\mathbf{U}} \, \log(4N\delta^{-1}) } +  \log(4N\delta^{-1})
\right) \right)
%\max_{k\in [K]}\left[ \mathrm{trace}(\Sigmabf_k)  + C\left(\sqrt{\|\Sigmabf_k\|_{\rm F}^2 (\log(4TK\ntrain\delta^{-1}))} + \opnorm{\Sigmabf_k}\ (\log(4TK\ntrain\delta^{-1})\right)\right]
\right\rbrace,
\end{align}
and
\begin{align}
\tau(\mathbf{V}) = \inf\left\lbrace n\geq 0\,:\,  
%\min_{t\in [T]}
\norm{P_{\mathbf{V}}(\x_{t,n})}^2 \geq 
\max_{k\in [K]} \left\lbrace \opnorm{ \Sigmabf_{k} } \right\rbrace 
\left( d_{\mathbf{V}}  + C\, C_{\mathbf{z}}^2\left( \sqrt{  d_{\mathbf{V}} \, \log(4N\delta^{-1}) } +  \log(4N\delta^{-1})
\right)\right) 
%\max_{k\in [K]}\left[ \mathrm{trace}(\Sigmabf_k)  + C\left(\sqrt{\|\Sigmabf_k\|_{\rm F}^2 (\log(4TK\ntrain\delta^{-1}))} + \opnorm{\Sigmabf_k}\ (\log(4TK\ntrain\delta^{-1})\right)\right]
\right\rbrace,
\end{align}

Again by definition of the Trace-Norm bandit, $\tau(\mathbf{U})$ and $\tau(\mathbf{V})$ are both stopping times relative to the filtration $\{\Fbar_{n}\}_{n\geq 0}$, so is $\tau=\tau(\mathbf{U})\wedge \tau(\mathbf{V})$. Hence the stopped process $\{\overline{\mathbf{M}}^{\tau}_n\}_{n\geq 0}$ defined as follows
$$
\overline{\mathbf{M}}^{\tau}_n:= P_{\mathbf{U}} \mathbf{M}_{t,n\wedge \tau} P_{\mathbf{V}},\quad \forall n\geq 0,
$$
is also a martingale adapted to the filtration $\{\Fbar_{n}\}_{n\geq 0}$. Furthermore $\{\overline{\mathbf{M}}^{\tau}_n\}_{n\geq 0}$ is a bounded martingale. Hence we can apply the Freedman inequality to it.

By definition of $\tau$ and \eqref{eq:OmegabarNter}, we have $\{\tau > N\}= \overline{\Omega}_N$. Hence for any $x>0$
\begin{align}
\label{eq:intermbis2}
\mathbb{P}\left( \opnorm{P_{\mathbf{U}} \mathbf{M}_{t,n} P_{\mathbf{V}}} \geq x \right) &= \mathbb{P}\left( \{ P_{\mathbf{U}} \mathbf{M}_{t,n} P_{\mathbf{V}} \geq x \} \cap \{\tau > N\}  \right) + \mathbb{P}\left( \{ \opnorm{P_{\mathbf{U}} \mathbf{M}_{t,n} P_{\mathbf{V}}} \geq x \} \cap \{\tau \leq N\}  \right) \notag\\
&\leq \mathbb{P}\left(  \opnorm{\overline{\mathbf{M}}^{\tau}_n} \geq x \right) + \mathbb{P}\left( \{\tau \leq N  \}  \right) \leq \mathbb{P}\left(  \opnorm{\overline{\mathbf{M}}^{\tau}_n} \geq x  \right) +\delta/2,
\end{align}
where we have used \eqref{eq:OmegabarNter} in the last line.

\paragraph{Deviation bound for $ \opnorm{\overline{\mathbf{M}}^{\tau}_n}$.}
We will use again \citep[][Corollary 1.3]{tropp2011freedman}.
%Theorem \ref{thm:freedman}.

In view of \eqref{eq:proofopnorm0}, \eqref{eq:OmeganterU} and \eqref{eq:OmeganterV}, we have 
\begin{align}
\opnorm{\overline{\mathbf{M}}^{\tau}_{n} - \overline{\mathbf{M}}^{\tau}_{n-1}} &\leq  \max_{k\in [K]} \left\lbrace \opnorm{\Sigmabf_{k} } \right\rbrace\left(  1+ (1+C\,C_{\mathbf{z}}^2/2)  
\sqrt{ \left( d_{\mathbf{U}} + \log(4N\delta^{-1}) \right)\left( d_{\mathbf{V}} + \log(4N\delta^{-1}) \right) } \right).
\end{align}
Next, we set $Y_{t,n} =  \overline{\mathbf{M}}^{\tau}_{n} - \overline{\mathbf{M}}^{\tau}_{n-1} 
%P_{\mathbf{U}}(\x_{t,n}) P_{\mathbf{V}}(\x_{t,n})^\top -  \E\left[P_{\mathbf{U}}(\x_{t,n}) P_{\mathbf{V}}(\x_{t,n})^\top \Big|\Fbar_{n-1}\right]
$. Elementary computations give
\begin{align*}
\bm{W}_{{\rm col}}&= \sum_{s=1}^n 
\mathbb{E}\left[ Y_{t,n} Y_{t,n}^\top \vert \Fbar_{s-1}\right] \leq (1+C\,C_{\mathbf{z}}^2/2)^2 n \bigl(  d_{\mathbf{V}} + \log(4N\delta^{-1}) \bigr)\max_{k\in [K]}\{\opnorm{P_{\mathbf{V}}\Sigmabf_{k}P_{\mathbf{V}}}\}\max_{k\in [K]}\{\opnorm{P_{\mathbf{U}}\Sigmabf_{k}P_{\mathbf{U}}}\},
\end{align*}
and
\begin{align*}
\bm{W}_{{\rm row}}&= \sum_{s=1}^n 
\mathbb{E}\left[  Y_{t,n}^\top Y_{t,n} \vert \Fbar_{s-1}\right] \leq (1+C\,C_{\mathbf{z}}^2/2)^2 n \bigl(  d_{\mathbf{U}} + \log(4N\delta^{-1}) \bigr)\max_{k\in [K]}\{\opnorm{P_{\mathbf{V}}\Sigmabf_{k}P_{\mathbf{V}}}\}\max_{k\in [K]}\{\opnorm{P_{\mathbf{U}}\Sigmabf_{k}P_{\mathbf{U}}}\}.
\end{align*}
Hence
\begin{align}
\label{eq:freedcond2bis}
\bm{W}_{{\rm row}}\vee \bm{W}_{{\rm col}} &\leq (1+C\,C_{\mathbf{z}}^2/2)^2 n \bigl(  d_{\mathbf{U}} \vee d_{\mathbf{V}}+ \log(4N\delta^{-1}) \bigr)\max_{k\in [K]} \{\opnorm{\Sigmabf_{k}}^2\}.
\end{align}
Applying \citep[][Corollary 1.3]{tropp2011freedman}, we get for any $\delta\in (0,1)$ with probability at least $1-\delta/(2N)$ that
\begin{align}
    \opnorm{\overline{\mathbf{M}}^{\tau}_n} &\leq  (1+C\,C_{\mathbf{z}}^2/2)\max_{k\in [K]} \{\opnorm{\Sigmabf_{k}}\} \biggl(  \sqrt{n\,\bigl(  d_{\mathbf{U}} \vee d_{\mathbf{V}}+ \log(4N\delta^{-1}) \bigr)\log(4N\delta^{-1}d) } \notag\\
    &\hspace{5cm} +2\sqrt{\left( d_{\mathbf{U}} + \log(4N\delta^{-1}) \right)\left( d_{\mathbf{V}} + \log(4N\delta^{-1}) \right) }\log(4N\delta^{-1}d)  \biggr).
\end{align}
Dividing by $n$ and a trivial union bound gives the result.

\end{proof}

\section{%\textcolor{red}{KARIM: I am trying another argument for the p
PROOF OF THEOREM \ref{Th:MTLRegretBound}}\label{AppSec:RegretUBound}

We recall that $\W = [\w_1,\dots,\w_T]$ is of rank smaller than $r$. Hence $\W$ admits the following singular value decomposition
$$
\W = \sum_{j=1}^{r} \sigma_j(\W) u_j(\W)\otimes v_j(\W),
$$
with singular values $\sigma_1(\W)\geq \cdots\geq \sigma_r(\W)>0$ and corresponding left and right orthonormal singular vectors $u_1(\W),\ldots,u_r(\W)$ of $\mathbb{R}^d$ and $v_1(\W),\ldots,v_r(\W)$ of $\mathbb{R}^T$. We denote by $P_{U}$ the orthogonal projection onto $U = \mathrm{l.s.}(u_1(\W),\ldots,u_r(\W))$. By definition of the SVD, we obviously have $P_U(\w_t) = \w_t$ for any $1\leq t \leq T$.

%\medskip
We consider now the instantaneous multi-task regret
\begin{equation*}
    \Rbar_n = \sum_{t=1}^T R_{t,n} = \sum_{t=1}^T  \langle \x_{t,n}^*,  \w_t \rangle  - \langle \x_{t,n} , \w_t \rangle  \leq 2 \sum_{t=1}^{T}\max_{\x\in\Dcal_{t,n}} |\langle \x, \w_t\rangle|,
\end{equation*}
where $\x_{t,n}^* = \arg\max_{\x\in\Dcal_{t,n}}\langle \x, \w_t \rangle$.
 
Assumption \ref{Ass:BoundedNorms} guarantees that $\frac{\langle \x,  \w_t \rangle}{\norm{\w_t }}$ is $C_{\x}$-subgaussian and consequently, for any $\delta\in (0,1)$, with probability at least $1-\delta$
$$
\vert \langle \x,  \w_t\rangle\vert \leq C_{\x} \norm{\w_t} \sqrt{\log(e\delta^{-1})} \leq C_{\mathbf{z}} \max_{k \in [K]} \left\lbrace \opnorm{\Sigma_k}^{1/2} \right\rbrace  L \sqrt{\log(e\delta^{-1})}.
$$
An union bound over $t\in [T]$, $n\in [N]$, $\x\in\Dcal_{t,n}$ gives with probability at least $1-\delta/2$
$$
\max_{\x\in\Dcal_{t,n}} |\langle \x, \w_t\rangle| \leq C_{\mathbf{z}} \max_{k \in [K]} \left\lbrace \opnorm{\Sigma_k}^{1/2} \right\rbrace  L \sqrt{\log(2eTKN\delta^{-1})},\quad \forall t\in [T],\; \forall k\in [K],\forall n\in [N].
$$
Hence with the same probability, $\forall n\in [N]$
$$
\Rbar_n \leq  C_{\mathbf{z}} T \max_{k \in [K]} \left\lbrace \opnorm{\Sigma_k}^{1/2} \right\rbrace  L \sqrt{\log(2eTKN\delta^{-1})}
$$

We will now split the regret into two phases. During the first $N_0$ interactions the collected data will not meet the RSC condition. Hence, the incurred regret will be of order $T N_0 \sqrt{\log(2eTKN\delta^{-1})}$.
%$L T N_0\left(r \bigvee \log(2TKN)\right)^{1/2}$.
Conversely, from the $N_0+1$-th round up to the $N$-th one, the RSC condition will be met allowing the result of Lemma \ref{Le:OracleInequality}.

We have
\begin{align}
\label{eq:regthm1}
    \sum_{n=\lceil N_0\rceil }^{N}\Rbar_{n} &= \sum_{n=\lceil N_0\rceil}^{N}\sum_{t=1}^T  \langle \x_{t,n}^*,  \w_t - \what_{t,n} \rangle + \sum_{n=\lceil N_0\rceil}^{N}\sum_{t=1}^T  \langle \x_{t,n}^* - \x_{t,n}, \what_{t,n} \rangle  + \sum_{n=\lceil N_0\rceil}^{N}\sum_{t=1}^T  \langle \x_{t,n} , \what_{t,n}-\w_t \rangle\notag\\
    &\leq  \sum_{n=\lceil N_0\rceil}^{N}\sum_{t=1}^T  \langle \x_{t,n}^*,  \w_t - \what_{t,n} \rangle  + \sum_{n=\lceil N_0\rceil}^{N}\sum_{t=1}^T  \langle \x_{t,n} , \what_{t,n}-\w_t \rangle
     = \sum_{n=\lceil N_0\rceil}^{N}\sum_{t=1}^T  \langle \x_{t,n}^* - \x_{t,n},  \w_t - \what_{t,n} \rangle
\end{align}

where we have used that for any $1\leq n \leq N$, $1\leq t\leq T$,  $\langle \x_{t,n}^* - \x_{t,n}, \what_{t,n} \rangle\leq 0$ a.s. by definition of our policy.

\medskip

Next, we have
\begin{equation}
    \label{eq:CSthm1}
\left\vert\sum_{t=1}^T  \langle \x_{t,n}^* - \x_{t,n},  \w_t - \what_{t,n} \rangle\right\vert 
%\leq 2 \max_{\x\in \Dcal_{t,n} }\left\vert\sum_{t=1}^T  \langle \x,  \w_t - \what_{t,n} \rangle\right\vert 
\leq  2 \sum_{t=1}^T\max_{\x\in \Dcal_{t,n} }\left\vert  \langle \x,  \w_t - \what_{t,n} \rangle\right\vert.
%\norm{P\X_{n}^* -P\X_{n}}_{\rm F}  \norm{ \What_{n} - \W}_{\rm F}\leq (\norm{P\X_{n}^*}_{\rm F} + \norm{P\X_{n}}_{\rm F} ) \norm{ \What_{n} - \W}_{\rm F},
\end{equation}

In view of \eqref{Eq:TraceNormReg}, we note that $\What_{n}\in \Fbar_{n-1}$. Thus, conditionally on $\Fbar_{n-1}$, Assumption \ref{Ass:BoundedNorms} guarantees that $\frac{\langle \x,  \w_t - \what_{t,n} \rangle}{\norm{\w_t - \what_{t,n} }}$ is $C_{\x}$-subgaussian and consequently, for any $\delta\in (0,1)$, with probability at least $1-\delta$
$$
\vert \langle \x,  \w_t - \what_{t,n} \rangle\vert \leq C_{\x} \norm{\w_t - \what_{t,n} } \sqrt{\log(e\delta^{-1})} \leq C_{\mathbf{z}} \max_{k \in [K]} \left\lbrace \opnorm{\Sigma_k}^{1/2} \right\rbrace \norm{\w_t - \what_{t,n} } \sqrt{\log(e\delta^{-1})}.
$$
An union bound over $t,n,k$ gives with probability at least $1-\delta$
\begin{align}
\label{eq:CSthm1-2}
   \max_{t\in [T],n\in [N]} \max_{\x\in \Dcal_{t,n} } \vert \langle \x,  \w_t - \what_{t,n} \rangle\vert \leq C_{\mathbf{z}} \max_{k \in [K]} \left\lbrace \opnorm{\Sigma_k}^{1/2} \right\rbrace \norm{\w_t - \what_{t,n} } \sqrt{\log(eTKN\delta^{-1})}.
\end{align}
Combining the last two displays, we get with probability at least $1-\delta$, for any $t\in [T]$, $n\in [N]$,
\begin{align*}
\left\vert\sum_{t=1}^T  \langle \x_{t,n}^* - \x_{t,n},  \w_t - \what_{t,n} \rangle\right\vert &\leq 2 C_{\mathbf{z}} \max_{k \in [K]} \left\lbrace \opnorm{\Sigma_k}^{1/2} \right\rbrace \left( \sum_{t=1}^{T} \norm{\w_t - \what_{t,n} } \right) \sqrt{\log(eTKN\delta^{-1})}\\
&\leq 2 C_{\mathbf{z}} \max_{k \in [K]} \left\lbrace \opnorm{\Sigma_k}^{1/2} \right\rbrace \sqrt{T} \norm{\W - \What_{n} }_F \sqrt{\log(eTKN\delta^{-1})}.
\end{align*}

Recall that Lemmas \ref{Le:OracleInequality} and \ref{Le:MatrixConcentration} combined with Propositions \ref{prop:subgaussian2} and \ref{Le:FirstN0} gives, for any $\lceil N_0\rceil \leq n\leq N$ and any $\delta\in (0,1)$, with probability at least $1-\frac{\delta}{2N}$, 
\begin{align}
\label{eq:intermediaire-WhatW}
\norm{ \What_{n} - \W}_{\rm F} \lesssim_{c_\eta,\sigma,C_{\x},\kappa(\Sigmabar),\max_{1\leq k \leq K }\left\lbrace \opnorm{\Sigma_{k}}^{1/2}\right\rbrace} 
%\sqrt{\frac{r}{n}} \left(  \sqrt{(T+d)} \bigvee  \sqrt{\log(4N\delta^{-1})}\right).
\sqrt{\frac{r}{n}} \left(  \sqrt{T+d+\log\left(\frac{4N(d+T)}{\delta}\right)}\log^{3/2}\left(\frac{8N(d+T)}{\delta}\right)\right).
\end{align}

Combining the last two displays, we obtain 
 with probability at least $1-\delta$
\begin{align}
\left\vert\sum_{t=1}^T  \langle \x_{t,n}^* - \x_{t,n},  \w_t - \what_{t,n+1} \rangle\right\vert &\lesssim   \sqrt{\frac{rT}{n}} \left(  \sqrt{T+d+\log\left(\frac{4N(d+T)}{\delta}\right)}\log^{3/2}\left(\frac{8N(d+T)}{\delta}\right)\right)\sqrt{\log(2e\delta^{-1}TKN) }.
\end{align}
%\left( r \bigvee (\log(4\delta^{-1}TKN^2))\right) 
Summing over $n$ in the previous display along with an union bound argument and  \eqref{eq:regthm1} gives with probability at least $1-\delta$
\begin{align}
 \sum_{n=\lceil N_0\rceil }^{N}\Rbar_s &\leq C \sqrt{rTN}    \left(  \sqrt{T+d+\log\left(\frac{4N(d+T)}{\delta}\right)}\log^{3/2}\left(\frac{8N(d+T)}{\delta}\right)\right)\sqrt{\log(2e\delta^{-1}TKN) },
\end{align}
where $C = C\left(\eta,\sigma,C_{\mathbf{z}},\kappa(\Sigmabar),\max_{1\leq k \leq K }\left\lbrace \opnorm{\Sigma_{k}}^{1/2}\right\rbrace\right)$ is a finite constant under our assumptions.

\section{NUMERICAL EXPERIMENTS}
\label{AppSec:Exp}
We perform here some additional experiments on the impact of the rank $r$, the dimension $d$ and the number of tasks $T$ on the performance of our trace norm bandit policy (Algorithm \ref{Alg:ASNucRegBan}). We also compare its performance to the MLingreedy policy of \cite{yang2020impact}.

\paragraph{Comparison to MLingreedy \citep{yang2020impact}.}

We compare the performance of the trace norm bandit (Alg. \ref{Alg:ASNucRegBan}) to that of the MLingreedy policy and independent task learning (ITL). We recall that the MLingreedy policy use a rank parameter $\overline{r}$. This parameter is set equal to the true rank $r$ of $\W$ in \cite{yang2020impact} to derive the theoretical properties of this policy. 

In our experiments we will use different values for the rank parameter $\overline{r}$, either the true rank $r$ or an overestimated rank $\overline{r} =\min( 2 r, d,T)$ or an underestimated rank $\overline{r} = \max(\lfloor r/2 \rfloor,1)$. At each epoch of the the MLingreedy policy, we used stochastic gradient descent to solve the matrix factorization step. We perform $5$ repetitions of our experiments for several values of dimension $d$, $T$, the parameter $\overline{r}$ of MLingreedy and the regularization parameter $l \in \{l_1,l_2,l_3\}$ of the trace norm policy.

In Figures \ref{Fig:MLingreedyd10best} and 
\ref{Fig:MLingreedyd40best}, we implement MLingreedy with the true value of rank $r$ (most favorable case for this policy in theory) and compare its performance to that of our trace norm bandit and ITL. Trace norm bandit performs consistently above ITL and MLingreedy with true parameter for all considered configurations of $T$ and $d$. Noticeably MLingreedy with true parameter performs significantly worse than ITL for $d=40$ for almost all values of $N$.

In Figure
\ref{Fig:MLingreedyd40}, we explore in more details the impact of parameters $l$ and $\overline{r}$ on the performance of the trace norm bandit and MLingreedy,  respectively. 

\begin{itemize}

\item 
In higher dimension $d=40$ and $T=10$, the trace norm bandit performs significantly better than MLingreedy. Actually MLingreedy performance is significantly worse than ITL even when using the true rank ($\overline{r}=r$) until $N$ becomes larger than $38$ whereas the trace norm policy performs significantly better than ITL even when $N$ is small. 
%This is in agreement with the theoretical guarantees which requires $N\geq d^2$ rounds for MLingreedy whereas trace norm requires only $N\gtrsim d$.
\item For $d=40$ and $T=30$, the trace norm bandit performs better than ITL for any $N\geq 25$ whereas MLingreedy with overestimated rank and true rank are below ITL until $N\geq 35$ and $N\geq 38$ respectively. MLingreedy with underestimated rank is significantly worse than ITL for all $N$.
\end{itemize}

\begin{figure}[t!] 
    \includegraphics[width=.49\textwidth]{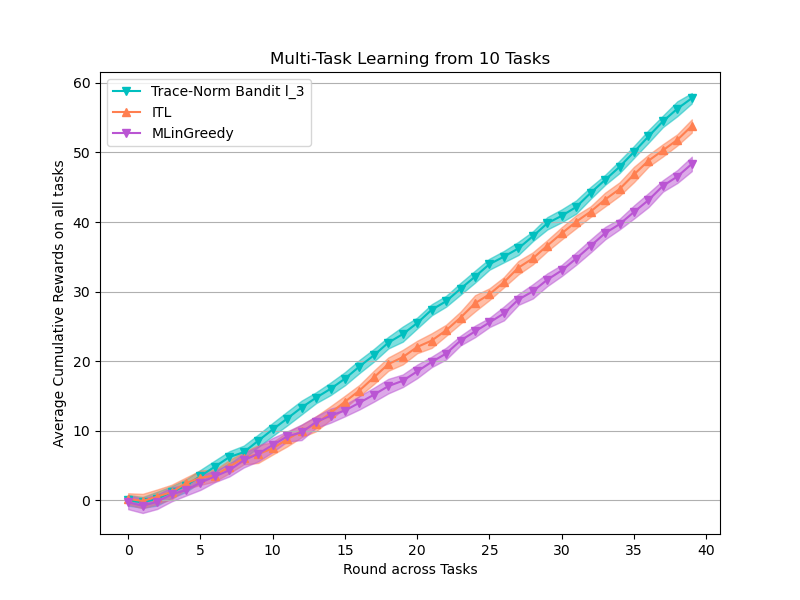}
    \includegraphics[width=.49\textwidth]{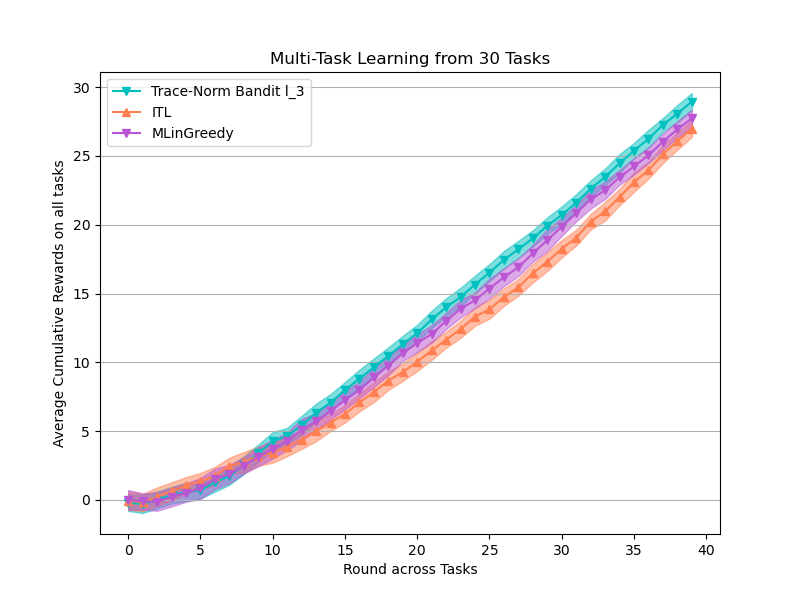}
    \caption{Averaged cumulative reward over all tasks with $d=10$ features, $T=10$ tasks (left) and $T=30$ (right). Each task lasts for $N=40$ rounds, has $K=10$ arms, noise variance $\sigma^2=1$ and true rank $r=5$.}
    \label{Fig:MLingreedyd10best}
\end{figure}

%(top) and $d=40$ (bottom) 

\begin{figure}[t!] 
    \includegraphics[width=.49\textwidth]{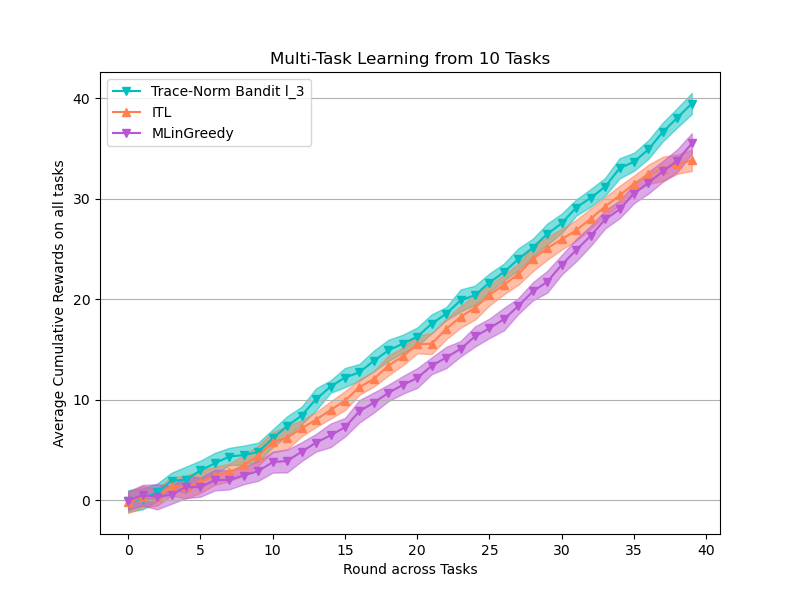}
    \includegraphics[width=.49\textwidth]{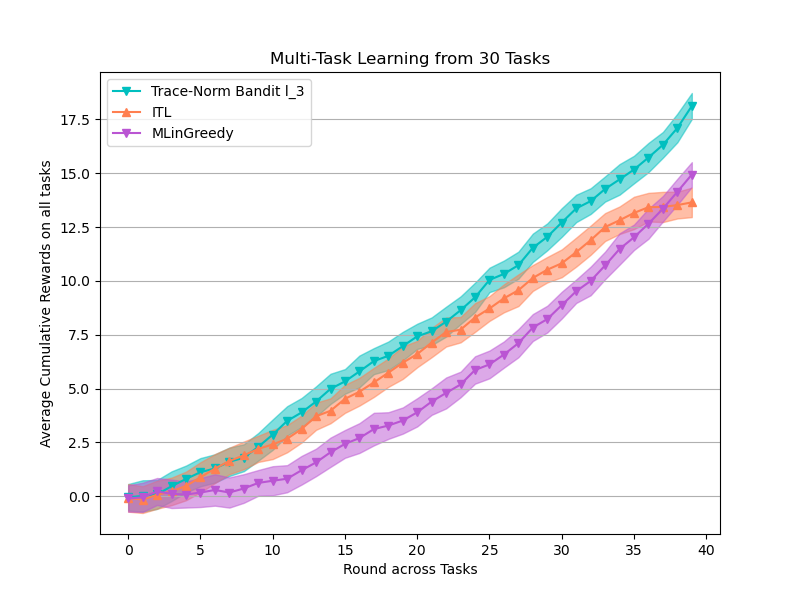}
    \caption{Averaged cumulative reward over all tasks with $d=40$ features, $T=10$ tasks (left) and $T=30$ (right). Each task lasts for $N=40$ rounds, has $K=10$ arms, noise variance $\sigma^2=1$ and true rank $r=5$.}
    \label{Fig:MLingreedyd40best}
\end{figure}

\begin{figure}[t!] 
    \includegraphics[width=0.49\textwidth]{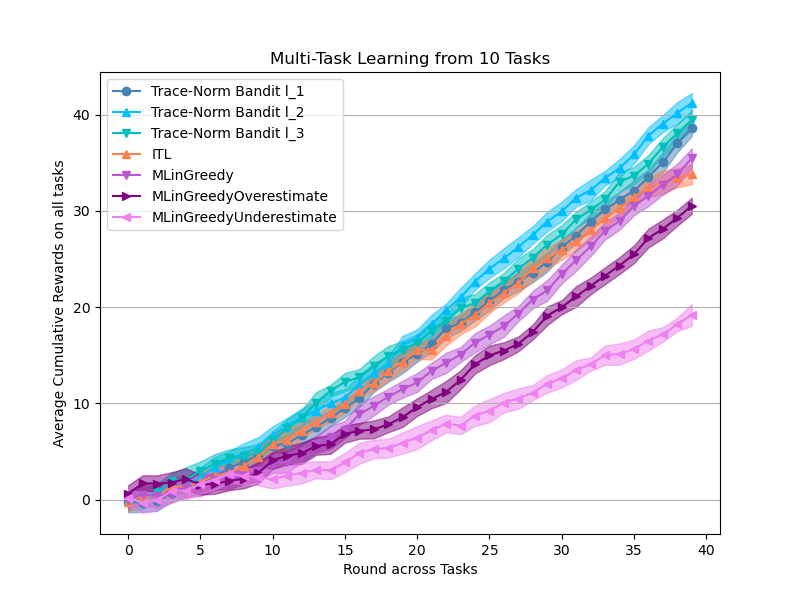}
    \includegraphics[width=0.49\textwidth]{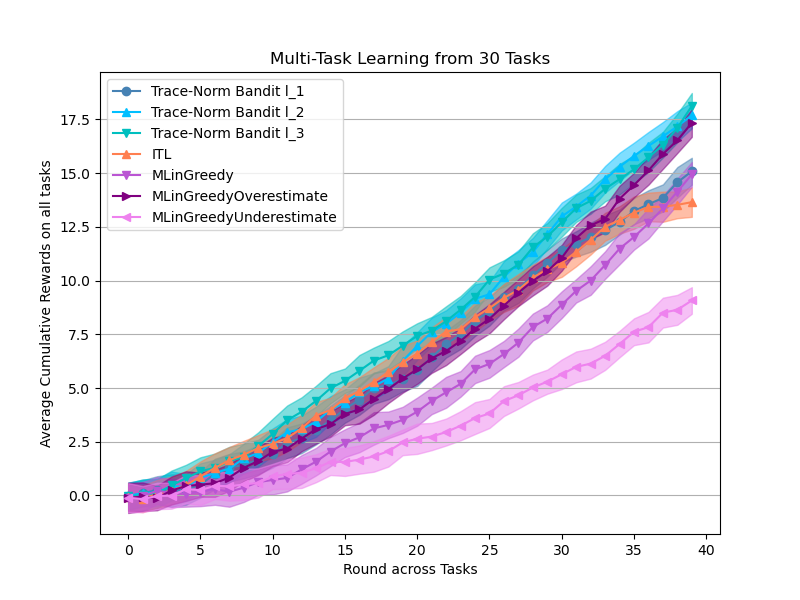}
    \caption{Averaged cumulative reward over all tasks with $d=40$ features, $T=10$ tasks (top) and $T=30$ (bottom). Each task lasts for $N=40$ rounds, has $K=10$ arms, noise variance $\sigma^2=1$ and true rank $r=5$.}
    \label{Fig:MLingreedyd40}
\end{figure}

In summary, the trace norm bandit performs consistently better than ITL uniformly for most values of $N$, for all values of $d$ and $T$ and for several choices of the regularization parameter. Conversely, MLingreedy is far more sensitive to the choice of rank parameter $\overline{r}$ and the values of $d$ and $T$. MLingreedy with underestimated rank performs far worse than ITL in all configurations. MLingreedy with overestimated rank performs sometimes better than MLingreedy with true rank. Overall, its performance are worse than those of trace norm bandit.

\paragraph{Impact of rank $r$.} In Figure \ref{Fig:impactr}, we note that Alg. \ref{Alg:ASNucRegBan} performs uniformly better than the ITL policy uniformly over all values of rank $r\in [d]$. The performance of the oracle tends to that of ITL as $r$ tends to $d$. This is expected as the oracle working in a $r$-dimensional space loses the benefit of working in a small dimensional space when $r\approx d$. Interestingly, the performance of the trace norm bandit becomes superior to that of the oracle as $r$ gets close to $d$. A plausible explanation is that nuclear norm regularization can still bring some benefit even when $r\approx d$ as it can still perform dimension reduction.

\begin{figure}[t!]
\begin{center}
    \includegraphics[width=0.49\textwidth]{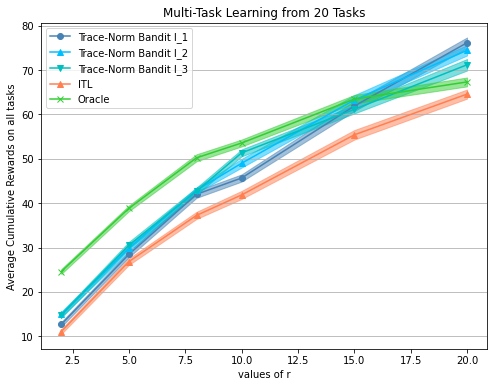}
    \includegraphics[width=0.49\textwidth]{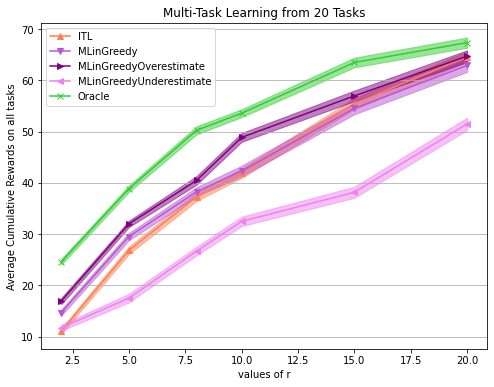}
    \caption{Averaged cumulative reward over all tasks for $d=20$, $T=20$, $\sigma^2=1$ as a function of the rank $r$.}
    \label{Fig:impactr}
\end{center}
\end{figure}

\end{document}